\documentclass{article}

\usepackage{iclr2027_conference,times}

\usepackage{iftex}
\ifPDFTeX
  \usepackage{times}
\else
  \usepackage{fontspec}
\fi

\usepackage{amsmath,amsfonts,bm}

\def\eqref#1{equation~\ref{#1}}

\def\1{\bm{1}}

\DeclareMathAlphabet{\mathsfit}{\encodingdefault}{\sfdefault}{m}{sl}
\SetMathAlphabet{\mathsfit}{bold}{\encodingdefault}{\sfdefault}{bx}{n}

\newcommand{\R}{\mathbb{R}}

\usepackage{amssymb,amsthm,mathtools}
\usepackage{xcolor}                 %
\usepackage{enumitem}
\usepackage{microtype}
\usepackage{booktabs}
\usepackage{colortbl}               %
\usepackage{tabularx}
\usepackage{multirow}
\usepackage{threeparttable}         %
\usepackage{pifont}                 %
\usepackage{graphicx}
\usepackage{subcaption}
\usepackage{float}                  %
\usepackage{algorithm}
\usepackage{algpseudocode}
\usepackage[hyphens]{url}
\definecolor{PaletteBlue}{HTML}{1F77B4}
\definecolor{PaletteTeal}{HTML}{17BECF}
\colorlet{CitationBlue}{PaletteBlue!55!black}
\colorlet{ReferenceTeal}{PaletteTeal!55!black}
\usepackage[colorlinks=true,linkcolor=ReferenceTeal,citecolor=CitationBlue,urlcolor=CitationBlue]{hyperref}
\newcommand{\dblstitchDoiLink}[1]{doi:~\href{https://doi.org/#1}{\nolinkurl{#1}}}
\newcommand{\dblstitchArxivLink}[1]{arXiv:\href{https://arxiv.org/abs/#1}{#1}}
\ExplSyntaxOn
\NewDocumentCommand{\doi}{m}{
  \str_set:Nn \l_tmpa_str {#1}
  \str_if_eq:eeTF { \str_range:Nnn \l_tmpa_str {1}{15} } { 10.48550/arXiv. }
    { \str_set:Nx \l_tmpb_str { \str_range:Nnn \l_tmpa_str {16}{-1} }
      \exp_args:NV \dblstitchArxivLink \l_tmpb_str }
    { \dblstitchDoiLink {#1} }
}
\ExplSyntaxOff
\usepackage[capitalize,noabbrev,nameinlink]{cleveref}

\makeatletter
\let\dblstitch@paragraph\paragraph
\renewcommand{\paragraph}[1]{\dblstitch@paragraph{#1.}}
\AtBeginDocument{%
  \let\dblstitch@addtocontents\addtocontents
  \renewcommand{\addtocontents}[2]{}%
}
\newif\ifdblstitch@firsttocsection
\newcommand{\appendixcontents}{%
  \dblstitch@firsttocsectiontrue
  \let\addtocontents\dblstitch@addtocontents
  \setcounter{tocdepth}{2}%
  \begingroup
  \renewcommand{\contentsname}{Appendix contents}
  \hypersetup{linkcolor=black}
  \setlength{\parskip}{0pt}
  \linespread{0.95}\selectfont
  \renewcommand*{\l@section}[2]{%
    \ifnum\c@tocdepth>\z@
      \ifdblstitch@firsttocsection\global\dblstitch@firsttocsectionfalse\else
        \addpenalty\@secpenalty\addvspace{0.5em \@plus\p@}%
      \fi
      \setlength\@tempdima{1.5em}%
      \begingroup
        \parindent\z@ \rightskip\@pnumwidth \parfillskip-\@pnumwidth
        \leavevmode\bfseries
        \advance\leftskip\@tempdima \hskip -\leftskip
        ##1\nobreak
        \leaders\hbox{$\m@th\mkern\@dotsep mu\hbox{.}\mkern\@dotsep mu$}\hfill
        \nobreak\hb@xt@\@pnumwidth{\hss ##2}\par
      \endgroup
    \fi}
  \tableofcontents
  \endgroup
  \medskip
}
\makeatother

\newcommand{\tablesize}{\fontsize{8}{9.6}\selectfont}
\newcommand{\appendixtablesize}{\fontsize{7}{8.4}\selectfont}
\newcommand{\cmark}{\ding{51}}
\newcommand{\xmark}{\ding{55}}
\newcommand{\tbd}{\textemdash}      %
\newcommand{\tbdnum}{\makebox[\widthof{$0.000$}]{\tbd}\phantom{${}\pm 0.000$}}%
\newcommand{\grouprule}{\arrayrulecolor{black!30}\specialrule{0.45pt}{1.5pt}{1.5pt}\arrayrulecolor{black}}
\newcommand{\grouplabel}[1]{\multirow{2}{*}{\begin{tabular}[c]{@{}l@{}}#1\end{tabular}}}
\newcommand{\inlinetablenotes}{\renewcommand\item[1][]{\ifhmode\unskip\ \fi\tnote{##1}\,\ignorespaces}}
\DeclareCaptionFont{fighead}{\small\bfseries}
\DeclareCaptionFont{figpanel}{\tablesize}
\newcommand{\AlgSubitem}[1]{\Statex \hspace{\algorithmicindent}\parbox[t]{\dimexpr\linewidth-\algorithmicindent\relax}{\hangindent=1.25em\hangafter=1\(\bullet\)\hspace{0.45em}#1}}
\algblockdefx[TRAIN]{TrainFor}{EndTrainFor}[1]{\hspace{\algorithmicindent}\textbf{for} #1 \textbf{do}}{\hspace{\algorithmicindent}\textbf{end for}}
\makeatletter\providecommand*{\theHALG@line}{\thealgorithm.\arabic{ALG@line}}\makeatother

\newtheoremstyle{paperplain}{1.5ex plus 0.5ex minus 0.2ex}{1.5ex plus 0.5ex minus 0.2ex}{}{0pt}{\bfseries}{.}{ }{\thmname{#1}\thmnumber{ #2}\thmnote{ (#3)}}
\theoremstyle{paperplain}
\newtheorem{theorem}{Theorem}
\newtheorem{proposition}{Proposition}
\newtheorem{corollary}{Corollary}
\newtheorem*{theoremrestate}{Theorem~\ref{thm:general-residual}}
\newtheorem*{propositionrestate}{Proposition~\ref{prop:complexity}}

\DeclarePairedDelimiter{\abs}{\lvert}{\rvert}
\DeclarePairedDelimiter{\norm}{\lVert}{\rVert}
\newcommand{\Ptwo}{\mathcal{P}_2(\mathbb{R}^d)}
\newcommand{\F}{\mathcal{F}}
\newcommand{\fdv}[2]{\frac{\delta #1}{\delta #2}}
\newcommand{\dd}{\,\mathrm{d}}
\newcommand{\grad}{\nabla}
\newcommand{\dive}{\grad\!\cdot}

\newif\ifTimePotentialResults
\TimePotentialResultstrue

\title{Simulation-Free Learning of Population Dynamics with Wasserstein Lagrangian Residuals}

\author{%
  \textbf{Fedor Sergeev}$^{1,2,3}$\thanks{Correspondence to fedor.sergeev@inf.ethz.ch} \quad
  \textbf{Markus Heinonen}$^{1}$ \quad
  \textbf{Daniel Waxman}$^{1,4}$ \quad
   \textbf{Tim Cooijmans}$^{1}$\\
  \textbf{Ricardo Baptista}$^{1,5}$ \,
  \textbf{Dmitry Batenkov}$^{1}$ \quad\,
  \textbf{Eli Bingham}$^{1}$ \\[6pt]
  \normalfont
  $^{1}$ Basis Research Institute 
  $^{2}$ ETH Zurich 
  $^{3}$ Helmholtz AI 
  $^{4}$ MIT
  $^{5}$ University of Toronto 
}

\iclrfinalcopy  %

\begin{document}
\maketitle

\begin{abstract}
    The dynamics of cells, organisms, and fluids are often modeled as probability distributions evolving over time. Reconstructing and extrapolating this evolution from unpaired snapshots requires assumptions about the underlying process. Wasserstein gradient flows are a common choice, but they cannot describe conservative or periodic dynamics. Lagrangian mechanics in Wasserstein space covers both, but existing methods for learning it are simulation-based: they run a numerical solver at every training step, which makes training expensive. We propose \emph{Double-Stitch}, a simulation-free method that learns these mechanics by penalizing the residual of the equation of motion along a learned population path. We derive this equation from a Clebsch variational principle that does not require gradient velocities, and show that the residual vanishes exactly when the equation holds. We test Double-Stitch on synthetic, single-cell and ocean vortex datasets and find that it matches or outperforms gradient-flow methods and simulation-based WLM on most tasks, while training $4$--$14$ times faster than WLM. We provide a JAX implementation of Double-Stitch at \href{https://github.com/BasisResearch/stitching}{https://github.com/BasisResearch/stitching}.
\end{abstract}

\section{Introduction}

    Learning complex dynamics of populations is a fundamental problem with applications to single-cell biology, animal behavior, and fluid mechanics \citep{schiebinger2019,maury2010,petrovic2025curlyfm}. Individual entities are usually not tracked: the population is observed as unpaired sets of samples (snapshots) measured at several time points. Generally, modeling approaches to this problem can be grouped into three families based on the assumptions they make about the data generating process: flow-based, gradient flow, and Lagrangian flow.

    Flow-based methods do not make explicit assumptions\footnote{Some methods can use a reference process, but it is designed for a specific dataset, not used as a general assumption \citep{petrovic2025curlyfm}.} about the underlying dynamical process \citep{kapusniak2024,atanackovic2025}. They learn a time-dependent velocity field or a stochastic bridge that transports an observed snapshot to a new time point by an optimal path (e.g., minimizing a kinetic energy). Without additional assumptions, these methods are generally not suited for extrapolation beyond the observed time interval \citep{guan2026wlm}.
    
    Gradient flow methods assume that the underlying dynamics follow the steepest descent of a free energy in Wasserstein space, or, in other words, a Wasserstein gradient flow (WGF). The resulting closed system allows learning the time-independent free energy function $\mathcal{F}$ and therefore provides time extrapolation capabilities. However, WGFs are purely dissipative and cannot effectively model dynamics where velocity is not deterministically defined by energy alone (e.g., as in periodic phenomena and conservative systems; see the conservative example in \cref{fig:gaussian-mixture}).

    The Lagrangian methods generalize gradient flows by assuming that the underlying dynamic process is governed by an equation on acceleration instead of velocity (i.e., Newton's second law). This approach was recently proposed for population dynamics \citep{guan2026wlm}, building on Lagrangian formulations in Wasserstein space \citep{neklyudov2024}. However, the existing algorithm (WLM) is simulation-based and computationally expensive because it simulates the trajectory at every training step.

    We propose \emph{Double-Stitch}, a simulation-free approach to learning Lagrangian mechanics. Unlike flow-based methods, this method allows extrapolation; unlike WGF-based methods, it models periodic and conservative systems (see \cref{fig:gaussian-mixture}); and it is computationally cheaper than the simulation-based WLM algorithm. \Cref{tab:method-comparison} summarizes the similarities and differences between the existing methods and our approach.

    Double-Stitch is based on Stitching \citep{heinonen2026}, a simulation-free WGF-based algorithm that learns the dynamics by directly optimizing the trajectory to minimize (a) the distance between the learned and observed snapshots, and (b) the residual of the underlying dynamical law. Double in the name refers to the second-order law, compared to the first-order equation for WGF. We re-derive the Wasserstein Lagrangian law without requiring gradient velocities and show its residual vanishes if and only if the dynamics are satisfied.

    \begin{figure}[t]
    \centering
    \includegraphics[width=\linewidth]{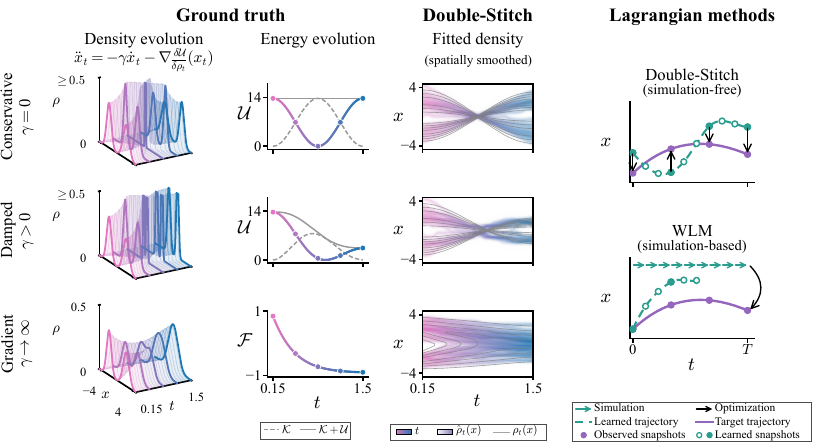}
    \caption{\textbf{Simulation-free fitting across three dynamical regimes.} Left: conservative, damped and gradient-flow examples, with analytic densities, energies and Double-Stitch fits. Right: schematic training procedures for Double-Stitch and WLM. Grey curves show kinetic and total energy for the oscillators; the gradient-flow panel shows free energy. Grey contours show analytic densities.}
    \label{fig:gaussian-mixture}
\end{figure}

    \paragraph{Contributions}
    \begin{itemize}[parsep=0pt]
        \item We propose a Clebsch-based derivation of the Wasserstein Lagrangian mechanics. This derivation does not restrict the velocities to be gradient, making it applicable to initial conditions with rotational components.
        \item We propose Double-Stitch, a simulation-free method that fits a latent population path and a neural potential by penalizing the discretized residual of the Lagrangian mechanics.
        \item We show that Double-Stitch matches or improves upon competing methods for most tasks on synthetic, single-cell, and ocean vortex datasets, while being cheaper to train than WLM.
    \end{itemize}

    \begin{table}[!htb]
\centering
\caption{Comparison of methods for learning population dynamics. Double-Stitch is a simulation-free method that extrapolates in time and learns both gradient and second-order Lagrangian flows.}
\label{tab:method-comparison}
\tablesize
\setlength{\tabcolsep}{4pt}
\begin{threeparttable}
    \begin{tabular*}{\linewidth}{@{\extracolsep{\fill}} l l l c c c c @{}}
        \toprule
        \multirow{2}{*}{Family} & \multirow{2}{*}{Method} & \multirow{2}{*}{Inputs} & \multirow{2}{*}{Extrapolation} & \multicolumn{2}{c}{Encoded flow} & \multirow{2}{*}{Simulation-free} \\
        \cmidrule(lr){5-6}
                                &                         &                         &                                & {\scriptsize Gradient} & {\scriptsize Lagrangian} & \\
        \midrule
        \grouplabel{Flow-based}
            & Flow matching, bridges & $X_i$                    & \xmark          & \xmark              & \xmark          & \cmark                    \\
            & Curly-FM               & $X_i$, reference process & \xmark\tnote{*} & \xmark              & \xmark\tnote{*} & \cmark                    \\
        \addlinespace[3pt]
        \grouplabel{Gradient flow}
            & JKOnet family          & $X_i$                    & \cmark          & \cmark              & \xmark          & \cmark/\xmark\tnote{\ddag} \\
            & Stitching              & $X_i$                    & \cmark          & \cmark              & \xmark          & \cmark                    \\
        \addlinespace[3pt]
        \grouplabel{Lagrangian flow}
            & WLM                    & $X_i,\ (v_0)$            & \cmark          & \cmark\tnote{\dag} & \cmark          & \xmark                    \\
            & Double-Stitch          & $X_i,\ (u_0)$            & \cmark          & \cmark\tnote{\dag} & \cmark          & \cmark                    \\
        \bottomrule
    \end{tabular*}
    \begin{tablenotes}[para]\scriptsize
        \inlinetablenotes
        \item[*] Reference process can be used for some extrapolation and flow encoding.
        \item[\dag] Recovered in the overdamped limit.
        \item[\ddag] Depending on the method.
    \end{tablenotes}
\end{threeparttable}
\end{table}

\section{Wasserstein Lagrangian mechanics}
\label{sec:mechanics}

\subsection{Problem setting}

    Consider the problem of learning the temporal evolution of a probability distribution $\rho: [0, T] \rightarrow \Ptwo$\footnote{Full notation is given in \cref{app:notation}.}. We assume that at certain time points $t_0 < \ldots < t_K$ we observe sets of independent samples from that distribution:
    \begin{equation*}
        X_i = \left\{x_i^{(n)}\right\}_{n=1}^{N_i}, \quad x_i^{(n)} \stackrel{\mathrm{ind.}}{\sim} \rho_{t_i}, \quad i = 0,\dots,K.
    \end{equation*}
    We will call these sets of samples \emph{snapshots}.
    The samples in consecutive snapshots need not be paired, i.e. $x_i^{(n)}$ and $x_{i+1}^{(n)}$ may describe distinct individuals.
    
    Our task is to create a model that, given the observed snapshots, will predict the underlying distribution $\hat \rho_t$ at a given time point $t$. The quality of the prediction is evaluated using the $p$-Wasserstein distance at selected times $\mathcal{T}_{\mathrm{eval}}$:
    \begin{equation*}
        \dfrac{1}{\abs{\mathcal{T}_{\mathrm{eval}}}} \sum_{t \in \mathcal{T}_{\mathrm{eval}}} W_p (\rho_t, \hat \rho_t).
    \end{equation*}
    We call the task \emph{interpolation} if $\max \left(\mathcal{T}_{\mathrm{eval}}\right) < t_K$ and \emph{extrapolation} otherwise.
    
    The dynamics of a continuous-time probability distribution are not identifiable from a finite set of snapshots \citep{weinreb2018,guan2026wlm}. To make the problem well-posed, additional assumptions on the underlying generating process are made.
    
\subsection{Dynamical assumptions}

    One common assumption is the conservation of mass, which leads to the \emph{continuity equation} \citep{ambrosio2008}:
    \begin{equation}
        \partial_t \rho_t + \dive \left(\rho_t u_t \right) = 0,
        \label{eq:continuity}
    \end{equation}
    where $u_t$ is the velocity. Although the continuity equation does constrain the applicable dynamics, it does not close the system because the velocity evolution is not defined.
    
    A large class of systems exhibits relaxation towards lowest free energy (e.g., diffusion). These processes can be formalized as \emph{Wasserstein gradient flows} (WGF), with their velocity following the steepest descent of the free energy \citep{jko1998,otto2001,heinonen2026}:
    \begin{equation}
        u_t = - \grad \fdv{\F}{\rho_t},
        \label{eq:wgf}
    \end{equation}
    where $\F$ is the free energy of the system. This assumption is used in the JKOnet family of methods and Stitching \citep{bunne2022,terpin2024jkonetstar,heinonen2026}.
    
    As a first-order condition with a time-independent free energy, WGF cannot describe systems that exhibit periodicity. In this case, a second-order assumption needs to be adopted. Here the velocity is no longer a variable that is defined by the particle positions, but can instead evolve independently.
    
    \citet{ambrosio2008} show that an absolutely continuous curve satisfying the continuity equation \cref{eq:continuity} admits multiple compatible velocities. Among these, the canonical velocity $v_t$ provides the solution with the minimum population kinetic energy. We define the canonical velocity through Helmholtz decomposition of the total velocity $u_t$ into a gradient field $v_t$ and a rotational field $w_t$ that moves mass without changing the density:
    \begin{equation}
        u_t = v_t + w_t, \qquad \dive (\rho_t w_t) = 0.
        \label{eq:helmholtz-decomp}
    \end{equation}

    \citet{guan2026wlm} use this canonical velocity to derive second-order dynamics from the stationary action of the Caldirola--Kanai (CK) Lagrangian (\cref{eq:background-mechanical-action}) for linear damping \citep{caldirola1941,kanai1948}:
    \begin{equation}
        \underbrace{\partial_t v_t + (v_t \cdot \grad) v_t}_{\text{acceleration}} = \underbrace{- \gamma v_t}_{\text{friction}} \underbrace{- \grad \fdv{\mathcal{U}}{\rho_t}}_{\text{potential force}}.
        \label{eq:wlm}
    \end{equation}
    The left-hand side is the acceleration of a particle moving with the flow. Here $\gamma$ is the friction coefficient and $\mathcal{U}$ is the potential energy. For autonomous systems $\gamma = 0$ corresponds to energy conservation, while $\gamma \rightarrow \infty$ recovers WGF (\cref{cor:caldirola-kanai}).
    
    The canonical velocity describes the instantaneous change in density, but need not coincide with the observed velocity. For second-order dynamics, however, velocity is part of the state and determines the subsequent motion. We therefore seek a residual for the full velocity $u_t$, allowing circulation.

\subsection{Equations of motion}
\label{sec:generalized}

    To derive the equations of motion for the full velocity $u_t$, we use the Clebsch action principle \citep{holm1998,cotter2007clebsch}. In addition to continuity, this formulation imposes transport of particle labels: each particle keeps its initial label as it moves. We use the Caldirola--Kanai action
    \begin{equation}
        \mathcal A=\int_0^T\mathrm{e}^{\gamma t}\big(\mathcal K-\mathcal U\big)\dd t,\qquad
        \mathcal K=\frac12\int\rho_t(x)\norm{u_t(x)}^2\dd x,
        \label{eq:background-mechanical-action}
    \end{equation}
    where $\mathcal U$ is the potential energy and $\gamma\ge0$ is a constant friction coefficient. The exponential factor produces linear damping, while the kinetic energy retains the full velocity, as in geometric hydrodynamics \citep{khesin2021}.

    \begin{theorem}[Lagrangian equations of motion]
    \label{thm:general-residual}
        Under the regularity assumptions in \cref{app:clebsch-derivation}, let $(\rho_t,u_t)$ satisfy continuity. For fixed $\mathcal U$ and $\gamma$, stationarity of \cref{eq:background-mechanical-action} under the Clebsch constraints, with density and labels fixed at the temporal endpoints, implies
        \begin{equation}
            \partial_tu_t+(u_t\cdot\grad)u_t
            =-\gamma u_t-\grad\fdv{\mathcal U}{\rho_t}.
            \label{eq:general-motion-law}
        \end{equation}
        The velocity $u_t$ need not be a gradient.
    \end{theorem}
    \begin{proof} See \cref{app:clebsch-derivation}.
    \end{proof}

    \begin{corollary}[Recovered dynamics]
        \label{cor:caldirola-kanai}
        A smooth solution of \cref{eq:general-motion-law} with a gradient initial velocity remains canonical and satisfies the WLM dynamics \cref{eq:wlm}. In this canonical class, setting $\gamma=0$ gives Wasserstein Hamiltonian flow (WHF) \citep{ambrosiogangbo2008,chow2020whf}. For a time-independent energy, zero damping conserves $\mathcal K+\mathcal U$, also for general velocities; the formal overdamped limit $\gamma\to\infty$, in slow time $\tau=t/\gamma$, recovers the Wasserstein gradient flow \cref{eq:wgf}.
    \end{corollary}
    \begin{proof} See \cref{app:recovered-flows}.
    \end{proof}

    \paragraph{Why general velocities}
    \label{sec:velocity-state}
    Observed particle velocities need not be canonical. Although replacing them by the canonical velocity preserves the instantaneous change in density, it can alter subsequent motion because velocity is part of the state in second-order dynamics. A residual evaluated along particle trajectories should therefore retain their full velocity. The Clebsch derivation provides this formulation, allowing observed velocities to be used directly as initial conditions without projecting them onto gradient fields. \Cref{app:velocity-comparison} illustrates how this projection can change the subsequent density evolution. The labels and the Lagrange multipliers (see \cref{app:clebsch-derivation}) are needed for the derivation and are not learned.

\section{Double-Stitch}
\label{sec:method}

    \paragraph{Lagrangian residual} We turn the equation of motion in \cref{eq:general-motion-law} into a loss by squaring its discrepancy, averaging over the population, and integrating over time. Along a particle trajectory satisfying $\mathrm{d}x_t/\mathrm{d}t=u_t(x_t)$, this gives
    \begin{equation}
        \mathcal R[\rho,u]
        =\mathbb E_{x_0\sim\rho_0}\!\left[\int_0^T
          \norm*{\ddot x_t+\gamma\dot x_t
           +\grad\fdv{\mathcal U}{\rho_t}(x_t)}^2\dd t\right].
        \label{eq:residual-pullback}
    \end{equation}
    Under the assumptions of \cref{thm:general-residual}, this nonnegative residual vanishes exactly when the equation of motion holds (and vice versa, see \cref{app:residual-theorem}). It can be evaluated directly on a fitted path without simulating the dynamics.

    \paragraph{Objective} Following Stitching, we learn the population dynamics by fitting the learned distribution curve $\hat \rho$ directly to the observed snapshots as well as to the assumed dynamics \citep{heinonen2026}. This means that we aim to minimize a loss with two components: \emph{data loss} $\mathcal{L}_{\mathrm{data}}$, responsible for fitting the observed snapshots, and \emph{residual loss} $\mathcal{L}_{\mathrm{res}}$, responsible for fitting the dynamics:
    \begin{equation}
        \mathcal L \triangleq \mathcal L_{\mathrm{data}} +\lambda\,\mathcal L_{\mathrm{res}},
        \label{eq:stitching-losses}
    \end{equation}
    where $\lambda$ is a balancing coefficient. The data loss is computed as a sum of distances between the observed and learned snapshots at the observation times $\mathcal T_{\mathrm{obs}}$
    \begin{equation*}
        \mathcal L_{\mathrm{data}}\triangleq \sum_{t_i\in\mathcal T_{\mathrm{obs}}}\mathcal D(\hat\rho^\theta_{t_i}, X_i),
    \end{equation*}
    where $\mathcal D$ is a divergence between the learned density and the observed samples, typically the forward KL (negative log-likelihood). The residual loss requires parametrization and discretization of the Lagrangian residual \cref{eq:residual-pullback}.

    \paragraph{Distribution parameterization}\label{sec:parametrization} We parametrize the distribution curve $\hat \rho \rightarrow \hat\rho^\theta$ using kernel density estimation (KDE):
    \begin{equation}
        \hat \rho_t^\theta(x) =\sum_{k=1}^P w_k^\theta\phi(x-x_{t,k}^\theta), \qquad \sum_{k=1}^P w_k^\theta=1.
        \label{eq:kde-path}
    \end{equation}
    Here $\phi$ is a Gaussian kernel, $x_{t,k}^\theta$ is the learnable center of kernel $k=1,\dots,P$ at time $t$, and $w_k^{\theta}$ is its weight in the overall mixture. We can then derive (see \cref{app:kde-velocity}) a velocity field satisfying continuity for this KDE path
    \begin{equation}
        u_t^\theta(x) =\sum_{k=1}^P r_{t,k}^\theta(x)\,\dot x_{t,k}^\theta, \qquad r_{t,k}^\theta(x) =\frac{w_k^\theta\phi(x-x_{t,k}^\theta)} {\sum_{l=1}^P w_l^\theta\phi(x-x_{t,l}^\theta)}, \qquad \sum_{k=1}^P r_{t,k}^\theta(x)=1,
        \label{eq:kde-velocity}
    \end{equation}
    where $r_{t,k}^{\theta}$ are the responsibilities of the kernels at $x$.
    
    \paragraph{Potentials and forces parametrization} We assume that the potential energy has the form 
    \begin{equation}
        \mathcal U =c_V\underbrace{\int_{\mathbb R^d}V(x)\,\dd\rho(x)}_{\text{external potential}} +c_H\underbrace{\int_{\mathbb R^d}\log\rho(x)\,\dd\rho(x)}_{\text{internal energy}} +\frac{c_W}{2}\underbrace{\int_{\mathbb R^d\times\mathbb R^d}W(x-x')\,\dd\rho(x)\,\dd\rho(x')}_{\text{particle interactions}}.
        \label{eq:stitching-functional}
    \end{equation}
    This form is common in applications \citep{villani,heinonen2026} and usefully has an analytic gradient of the variation with respect to paths $\grad \frac{\delta \mathcal{U}}{\delta \rho}$ (see \cref{eq:stitching-functional-gradient} in \cref{app:wasserstein-gradient}). We parameterize the external potential and particle interactions with neural networks $V = V_\theta$ and $W = W_\theta$ and compute the entropy term from the KDE approximation, resulting in a parameterized potential $\mathcal U_\theta[\rho]$. In Stitching, this functional is the free energy $\F$ of the gradient flow \cref{eq:wgf}, whose entropy term drives diffusion; here it is the potential energy $\mathcal U$ in the second-order law \cref{eq:general-motion-law}, and the same term acts as an isothermal pressure with force $-c_H\grad\log\rho$.
        
    \paragraph{Empirical objective} To compute the data loss numerically, we simply compute it for the learned KDE density $\hat\rho^\theta_t$ and the observed set of particles. 
    
    To compute the residual loss, we approximate the population average in \cref{eq:residual-pullback} by a weighted sum over the learned center paths. We evaluate density-dependent forces using the KDE and approximate the interaction convolution using the centers (\cref{app:residual-theorem}). The acceleration is the derivative along each center path; this particle approximation differs from the exact material acceleration of the KDE mean velocity, as detailed in \cref{app:kde-velocity}. Finally, we approximate the time integral by quadrature and the path derivatives by finite differences:
    \begin{equation}
        \mathcal L_{\mathrm{res}}(\theta) = h\sum_{j=0}^{M}\sum_{k} w_k \Big\|\Delta_h^2 x_{t_j,k} +\gamma\,\Delta_h x_{t_j,k} +\grad\fdv{\mathcal U_\theta}{\rho}\big[\hat\rho_{t_j}^\theta\big](x_{t_j,k})\Big\|^2,
        \label{eq:particle-second-order-residual}
    \end{equation}
    where $\Delta_h$ and $\Delta_h^2$ are the finite differences of the center positions on the time grid $t_0<\dots<t_M$ with spacing $h$, central at the interior nodes and one-sided at the two endpoints, and the terms at $j=0$ and $j=M$ are halved (trapezoidal rule). Here $\gamma$ is a learnable friction coefficient.

    \paragraph{Initial velocity constraint} When paired initial positions and velocities $\{(x_0^{(n)},v_0^{(n)})\}_{n=1}^{N_0}$ are available, we add $\lambda_{v_0}\mathcal L_{v_0}$ to the objective in \cref{eq:stitching-losses}. Using the KDE responsibilities from \cref{eq:kde-velocity} and a forward difference on the first grid interval, we set
    \begin{equation}
        \mathcal L_{v_0}(\theta)=\frac{1}{N_0}\sum_{n=1}^{N_0}\left\|\sum_{k=1}^{P}r_{t_0,k}^{\theta}(x_0^{(n)})\,\frac{x_{t_1,k}^{\theta}-x_{t_0,k}^{\theta}}{h}-v_0^{(n)}\right\|^2.
        \label{eq:initial-velocity-loss}
    \end{equation}
    This softly constrains the initial velocity field without requiring velocity observations at later times.
        
    \begin{proposition}[Training cost]\label{prop:complexity} Consider the training cost of one Double-Stitch and WLM training iteration on the same task. Suppose that the number of time steps $M$, number of particles $P$, and the data loss evaluation time are the same. Assume that $C(P)$ is the time it takes to make a force evaluation ($P$ particles, including backpropagation), and suppose the hardware can run $B$ force evaluations in parallel. Then, apart from the data loss, one Double-Stitch iteration takes $O(\lceil M/B\rceil\,C(P))$ time while WLM will take $O(M\,C(P))$.
    \end{proposition}
    \begin{proof}
        See \cref{app:proof-complexity}; it follows directly from the pseudocode of the algorithms in \cref{app:stitching-algorithm}.
    \end{proof}

\section{Experiments}
\label{sec:experiments}

    We first demonstrate the difference between the Lagrangian flow Double-Stitch and the gradient flow Stitching on synthetic datasets, then compare the performance of our method against published results on real single-cell and fluid dynamics datasets.
    
    We consider multiple prediction modes: direct, push-forward, and rollout. The direct method is applicable to particle-based methods (i.e., Stitching and Double-Stitch) and essentially scores the particles that are learned directly. Push-forward takes the latest observed snapshot and unrolls the trajectory until the required test time using the learned field. The rollout does the same, but unrolls from the earliest observed snapshot. 
    
    We provide additional visualizations, expanded method comparisons, demonstrations of non-identifiability, as well as sensitivity and ablation studies in \cref{app:additional-results}. 

\subsection{Synthetic demonstration}
\label{sec:demonstrations}

    First, to illustrate the benefit of second-order modeling, we consider three Gaussian-mixture evolutions: conservative harmonic motion, damped harmonic motion with $\gamma=1$, and Wasserstein gradient flow. The conservative flow reverses the mixture and conserves total energy. The damped flow is the same process, but the damping reduces the oscillation amplitude. The gradient flow relaxes towards a single Gaussian. \Cref{fig:gaussian-mixture} shows their densities, energy profiles and Double-Stitch fits (dataset details are given in \cref{app:data-details}).

    We train Stitching \citep{heinonen2026} and Double-Stitch with $50$ centers, $50$ time nodes and $5000$ updates on the same three snapshots, holding out two times. For a single seed, the test $W_1$ of Stitching versus Double-Stitch is $1.324$ versus $0.389$ on conservative flow, $0.711$ versus $0.366$ on damped flow, and $0.273$ versus $0.292$ on gradient flow. \Cref{fig:synthetic-density-comparison} in \cref{app:synthetic-results} visualizes ground truth and fitted densities produced by both methods.

\subsection{Single-cell dynamics}
\label{sec:eb}

    Second, we consider the embryoid-body (EB) single-cell RNA dataset \citep{moon2019}. It describes the evolution of human embryonic stem cells over 5 snapshots. 3 different training and evaluation protocols have been proposed for this dataset. First, train on all snapshots, holding out $30\%$ of the cells of each for evaluation \citep{terpin2024jkonetstar,persiianov2025}. Second, leave one snapshot out for validation \citep{tong2020}. Third, leave two snapshots out for validation and average the performance on them \citep{shen2025}. We refer to these protocols as leave-particles out (LPO), leave-one-out (L1O), and leave-two-out (L2O) respectively. Since the key methods we compare to, Stitching \citep{heinonen2026} and WLM \citep{guan2026wlm} use different protocols, we rerun them alongside our method for consistency. Here we score the particles learned by Double-Stitch directly and use pushforward evaluation for WLM (see dataset details in \cref{app:data-details} and training details in \cref{app:training-details}).

    \begin{table}[!htb]
    \centering
    \setlength{\tabcolsep}{4pt}
    \caption{\textbf{Double-Stitch matches performance of the best gradient and Lagrangian methods on the EB dataset} ($W_1$ and $W_2$, mean \textpm{} std, 5 seeds). Lower is better. Dashes indicate unreported cells. $V(x)$ and $V(x,t)$ denote static and time-dependent potentials. Stitching and WLM rows are our reruns (\cref{app:reproduction-audit}); \cref{app:eb-results} also lists their published values.}
    \label{tab:eb-comparison}
    \scriptsize
    \resizebox{\ifdim\width>\linewidth\linewidth\else\width\fi}{!}{%
    \begin{tabular}{@{}>{\scriptsize\itshape\raggedright\arraybackslash}p{3.2cm}lrrr@{}}
\toprule
 & Method & \multicolumn{1}{c}{LPO, $W_1$} & \multicolumn{1}{c}{L1O, $W_1$} & \multicolumn{1}{c}{L2O, $W_2$} \\
\midrule
\multirow{11}{=}{Transport, bridges \& neural flows} & Neural SDE & $0.820\phantom{{}\pm 0.000}$ & \tbdnum & \tbdnum \\
 & TrajectoryNet & $0.930\phantom{{}\pm 0.000}$ & $0.848\phantom{{}\pm 0.000}$ & $1.980\phantom{{}\pm 0.000}$ \\
 & SB-FBSDE & $0.840\phantom{{}\pm 0.000}$ & \tbdnum & \tbdnum \\
 & NLSB & $0.780\phantom{{}\pm 0.000}$ & $0.970\phantom{{}\pm 0.000}$ & \tbdnum \\
 & OT-CFM & $0.770\phantom{{}\pm 0.000}$ & $0.790 \pm 0.068$ & \tbdnum \\
 & WLF-OT & $0.740\phantom{{}\pm 0.000}$ & $0.742 \pm 0.012$ & \tbdnum \\
 & WLF-SB & $0.730\phantom{{}\pm 0.000}$ & $0.746 \pm 0.016$ & \tbdnum \\
 & Vanilla-SB & \tbdnum & \tbdnum & $1.520\phantom{{}\pm 0.000}$ \\
 & MMSB & \tbdnum & \tbdnum & $1.420\phantom{{}\pm 0.000}$ \\
 & DMSB & \tbdnum & \tbdnum & $1.290\phantom{{}\pm 0.000}$ \\
 & OT-MFM & \tbdnum & $0.713 \pm 0.039$ & \tbdnum \\
\grouprule
\multirow{5}{=}{WGF energy models} & JKOnet & $1.340\phantom{{}\pm 0.000}$ & \tbdnum & \tbdnum \\
 & JKOnet$^*$, $V(x)$ & $1.120\phantom{{}\pm 0.000}$ & \tbdnum & $1.840\phantom{{}\pm 0.000}$ \\
 & JKOnet$^*$, $V(x,t)$ & $0.730\phantom{{}\pm 0.000}$ & \tbdnum & $3.590\phantom{{}\pm 0.000}$ \\
 & iJKOnet, $V(x)$ & $1.050\phantom{{}\pm 0.000}$ & \tbdnum & $1.120\phantom{{}\pm 0.000}$ \\
 & iJKOnet, $V(x,t)$ & $0.580\phantom{{}\pm 0.000}$ & \tbdnum & $0.920\phantom{{}\pm 0.000}$ \\
\grouprule
\multirow{2}{=}{Simulation-free WGF} & Stitching, $V(x)$ & $0.618 \pm 0.007$ & $0.796 \pm 0.014$ & $0.821 \pm 0.019$ \\
 & Stitching, $V(x,t)$ & $0.581 \pm 0.010$ & $0.717 \pm 0.010$ & $0.786 \pm 0.008$ \\
\addlinespace[5pt]
Simulation-based WLM & WLM & $0.682 \pm 0.009$ & $0.734 \pm 0.023$ & $0.823 \pm 0.041$ \\
\addlinespace[5pt]
\multirow{2}{=}{Simulation-free WLM} & Double-Stitch, $V(x)$ & $0.618 \pm 0.001$ & $0.733 \pm 0.004$ & $0.794 \pm 0.004$ \\
 & Double-Stitch, $V(x,t)$ & $0.555 \pm 0.005$ & $0.717 \pm 0.008$ & $0.799 \pm 0.003$ \\
\bottomrule
\end{tabular}
}
\end{table}

    We show the results in \cref{tab:eb-comparison} (flow-based and JKOnet models are adapted from \citet[Tables~2 and~5]{heinonen2026} and \citet[Table~3]{guan2026wlm}). We find that Double-Stitch outperforms or matches WLM using both static and time-variable potential. It also performs better or the same as Stitching for all cases except one (L2O for time dependent potential). Both comparisons use our reruns (see the settings in \cref{app:reproduction-audit}); \cref{app:eb-results} also lists the published values. Visualization of predictions and extended comparison tables are given in \cref{app:eb-results}.

\subsection{Ocean vortex dynamics}
\label{sec:oceans}

    Third, we evaluate the methods on two datasets of ocean vortex dynamics in the Gulf of Mexico \citep{guan2026wlm}. The first one captures a small vortex and provides an interpolation task: it consists of 9 snapshots of 111 particle positions, with even snapshots used for training and odd ones for testing. The second one captures a large vortex and provides a forecasting task: it consists of 11 snapshots of 400 particle positions and uses the final snapshot for testing (see data details in \cref{app:data-details}).

    Following \citet{guan2026wlm}, we provide our method with the initial velocity through the initial-velocity constraint introduced in \cref{eq:initial-velocity-loss}. The settings and evaluation protocols are described in \cref{app:training-details}. At evaluation, both methods receive the observed velocity in pushforward snapshots and the initial velocity for a full rollout.

    \begin{table}[!htb]
    \centering
    \setlength{\tabcolsep}{4pt}
    \caption{\textbf{Double-Stitch outperforms gradient methods and matches the best Lagrangian method on ocean vortex interpolation} ($W_1$, mean \textpm{} std, 3 seeds). Lower is better. Dashes indicate unreported cells. $V(x)$ and $V(x,t)$ denote static and time-dependent potentials. Stitching and WLM rows are our reruns (\cref{app:reproduction-audit}); \cref{app:oceans-results} also lists WLM's published values.}
    \label{tab:oceans-comparison}
    \scriptsize
    \resizebox{\ifdim\width>\linewidth\linewidth\else\width\fi}{!}{%
    \begin{tabular}{@{}>{\scriptsize\itshape\raggedright\arraybackslash}p{3.2cm}lrrr@{}}
\toprule
 & Method & \multicolumn{1}{c}{Interp. pushfwd.} & \multicolumn{1}{c}{Forecast pushfwd.} & \multicolumn{1}{c}{Forecast rollout} \\
\midrule
\multirow{4}{=}{Action \& continuity} & AM & $0.376\phantom{{}\pm 0.000}$ & \tbdnum & \tbdnum \\
 & UAM & $0.383\phantom{{}\pm 0.000}$ & \tbdnum & \tbdnum \\
 & sAM & $0.378\phantom{{}\pm 0.000}$ & \tbdnum & \tbdnum \\
 & DICE & $0.244\phantom{{}\pm 0.000}$ & \tbdnum & \tbdnum \\
\grouprule
\multirow{3}{=}{Transport \& bridges} & Vanilla-SB & $0.214\phantom{{}\pm 0.000}$ & \tbdnum & \tbdnum \\
 & OT-CFM & $0.204\phantom{{}\pm 0.000}$ & \tbdnum & \tbdnum \\
 & OT-MFM & $0.071\phantom{{}\pm 0.000}$ & \tbdnum & \tbdnum \\
\grouprule
\multirow{3}{=}{Reference process} & Curly-FM & $0.030\phantom{{}\pm 0.000}$ & \tbdnum & \tbdnum \\
 & SBIRR & $0.076\phantom{{}\pm 0.000}$ & $0.567 \pm 0.014$ & $1.062 \pm 0.111$ \\
 & Snap-MMD & $0.076\phantom{{}\pm 0.000}$ & $0.473 \pm 0.013$ & $0.896 \pm 0.141$ \\
\grouprule
\multirow{2}{=}{Simulation-free WGF} & Stitching, $V(x)$ & $4.750 \pm 0.535$ & $0.205 \pm 0.006$ & $0.149 \pm 0.007$ \\
 & Stitching, $V(x,t)$ & $0.104 \pm 0.020$ & $0.170 \pm 0.004$ & $0.211 \pm 0.092$ \\
\addlinespace[5pt]
Simulation-based WLM & WLM & $0.064 \pm 0.015$ & $0.088 \pm 0.023$ & $0.194 \pm 0.048$ \\
\addlinespace[5pt]
\multirow{2}{=}{Simulation-free WLM} & Double-Stitch, $V(x)$ & $0.065 \pm 0.004$ & $0.134 \pm 0.003$ & $1.839 \pm 0.539$ \\
 & Double-Stitch, $V(x,t)$ & $0.073 \pm 0.009$ & $0.119 \pm 0.019$ & $0.222 \pm 0.008$ \\
\bottomrule
\end{tabular}
}
\end{table}

    We present the results in \cref{tab:oceans-comparison}, where the results for flow-based and reference process methods are adapted from \citet[Table~2]{guan2026wlm} (see the visualizations and per-snapshot performance in \cref{app:oceans-results}). On interpolation, static Double-Stitch matches our WLM rerun and outperforms all other methods except Curly-FM, which relies on reference process. On forecasting, WLM performs best under pushforward, while Double-Stitch still improves on Stitching and reference process methods. Under rollout, static Double-Stitch performs poorly and time-dependent potential brings its error close to WLM, although static Stitching is the best here. Note that at $t=10$ the time-dependent potential is evaluated beyond the last observed snapshot, which assumes that its learned time dependence extrapolates.

\subsection{Computational cost}
\label{sec:computational-cost}

    Finally, we investigate the computational cost of our simulation-free algorithm and compare it to the simulation-based WLM. 
    
    First, we measure the practical training time on EB and Oceans (see \cref{tab:actual-cost}). Double-Stitch trains in roughly $4$--$12$ minutes on EB and $7$--$10$ minutes on oceans, while WLM takes $16$--$54$ and $97$--$145$ minutes on identical hardware (see the details in \cref{app:training-details}).

    Second, we match the parameters of the two methods (number of updates, time steps and data) and run them on the synthetic WHF Gaussian mixture flow from \cref{fig:gaussian-mixture}. We find that WLM takes almost $7.6$ times as long ($225$ vs $30$ seconds, \cref{fig:cost-efficiency}) to reach the same level of performance (test $W_1$ of about $0.09$).
    
    This illustrates the computational advantage of simulation-free training in this setting.

    \newsavebox{\costTablePanel}
\newsavebox{\costFigureCaption}
\newlength{\costPlotHeight}
\begin{figure}[!htb]
    \centering
    \sbox{\costTablePanel}{\begin{minipage}[t]{0.46\linewidth}
        \vspace{0pt}
        \input{tables/actual-cost}
    \end{minipage}}
    \sbox{\costFigureCaption}{\begin{minipage}[t]{0.515\linewidth}
        \vspace{0pt}
        \caption{\textbf{Training on synthetic WHF} (mean \textpm{} std, 3 seeds). Static methods. Arrow shows the training time ratio at test $W_1\approx0.09$.}
        \label{fig:cost-efficiency}
    \end{minipage}}
    \setlength{\costPlotHeight}{\dimexpr\ht\costTablePanel+\dp\costTablePanel-\ht\costFigureCaption-\dp\costFigureCaption+\baselineskip\relax}
    \typeout{Cost plot height: \the\costPlotHeight}
    \usebox{\costTablePanel}\hfill
    \begin{minipage}[t]{0.515\linewidth}
        \vspace{0pt}
        \centering
        \includegraphics[width=\linewidth,height=\costPlotHeight,keepaspectratio]{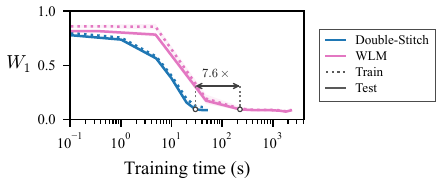}\par\nointerlineskip
        \usebox{\costFigureCaption}
    \end{minipage}
\end{figure}

\section{Related work}
\label{sec:related-work}
    
    \paragraph{Flow-based methods} Flow matching \citep{lipman2023,tong2024}, dynamic optimal transport \citep{tong2020}, and Schr\"odinger bridges \citep{vargas2021,chen2022,koshizuka2023,chen2023} learn a time-dependent velocity field or stochastic bridge that transports one snapshot into the next according to some optimality criterion (usually the minimal kinetic energy trajectory). Action matching \citep{neklyudov2023} fits the potential jointly to all marginals without a coupling; \citet{neklyudov2024} generalize the interpolant to any user-specified Lagrangian. As described above, without additional assumptions, these methods are generally not suitable for extrapolation beyond the last observed snapshot \citep{berlinghieri2025}. However, they can be extended to handle non-gradient flows by adding a reference process \citep{petrovic2025curlyfm,shen2025}. 
    
    \paragraph{Gradient flow methods} The JKOnet family of methods imposes WGF via a Jordan--Kinderlehrer--Otto (JKO) scheme, which reduces the problem to learning the potential function $\mathcal{F}$. Various methods offer different ways of learning it \citep{bunne2022,amos2017,persiianov2025}, but they all need to unroll the JKO scheme to get the population trajectory. Alternatively, WGF can be softly imposed via a Wasserstein residual, similar to physics-informed collocation losses \citep{raissi2019}. Stitching \citep{heinonen2026} directly optimizes the trajectory and learns the potential to fit the observed snapshots and minimize the residual. GenWGP \citep{liu2026} uses the same residual for a known energy but integrates a Neural ODE.
    
    \paragraph{Lagrangian flow methods} \citet{neklyudov2024} use a given Lagrangian in an action-matching framework. \citet{guan2026wlm} use a CK Lagrangian in Wasserstein space with Lagrange multipliers and propose a simulation-based algorithm for learning its energy and friction. We derive the same dynamical law using an alternative approach proposed in geometric hydrodynamics, specifically the Clebsch variational formulation \citep{cotter2007clebsch} of Euler--Arnold and Euler--Poincar\'e equations \citep{holm1998,khesin2021}.

    \paragraph{Residuals and numerics} Penalizing the residual of a governing equation has long been used in physics-informed neural networks \citep{raissi2019,karniadakis2021}. Their optimization is known to suffer from competing loss terms and stiff dynamics \citep{krishnapriyan2021,wang2021pinn}. Related works have enforced some constraints like conservation laws by construction \citep{richterpowell2022}. Our method uses both these approaches: dynamics are imposed softly via a residual while continuity is enforced by construction. We find that numerical methods offer useful language to compare methods for learning population dynamics: Stitching and Double-Stitch are collocation methods (directly optimize the trajectory), JKOnet and WLM are shooting methods (simulate the trajectory forward), and multiple shooting isn't represented in the field \citep{bock1984}.

\section{Discussion}

    \paragraph{Summary} We present Double-Stitch, a simulation-free residual-based method for learning population dynamics. We derive a Lagrangian dynamical law for general (non-gradient) velocities, which leads to a Wasserstein residual allowing learning gradient, conservative, and damped dynamics. We show that our method matches or outperforms Stitching and WLM on single-cell data and ocean interpolation while training $4$--$14$ times faster than WLM, although WLM remains better on ocean forecasting.
    
    \paragraph{Limitations} As a residual-based method and unlike simulation-based methods, Double-Stitch does not guarantee exact satisfaction of the assumed dynamical law. For the same reason, its optimization may be challenging if the initialization is poor or the residuals compete in a complex manner. %

    \paragraph{Future work} While we focus on the Caldirola--Kanai Lagrangian formulation, future work could investigate alternative Lagrangians (e.g., with Rayleigh dissipation functions) and consider more general mechanics such as active matter and viscous flows. Another line of research could develop hybrid methods that combine the advantages of Double-Stitch and WLM via multiple-shooting.

\subsection*{AI use statement}

    In this work, we used generative AI tools for creating and editing software code (including implementation of methods, plotting, table extraction, and references formatting), proposing keywords, editing the manuscript to improve readability (including drafting some paragraphs from raw notes), assistance in the writing of proofs and mathematical claims (specifically unifying notation, transcribing and adapting equations, verifying proof correctness), generating synthetic datasets (specifically finding the right parameter settings to make them visual and consistent). We have not used generative AI tools to clean and reformat datasets, discover research topics or identify gaps, assist with translation, suggest a structure for a research paper and support qualitative and thematic data analysis. Tasks such as formulating questions for surveys or interviews, creating artifacts, and transcribing recordings of research material are not applicable to this work. We have reviewed all AI-assisted work (derivations were checked manually and generated code was verified and tested). We take responsibility for the final content of this work, including text, claims, or artifacts produced with the aid of generative AI.

\subsection*{Ethics statement}

    This work focuses on machine learning research using publicly available datasets with applications to computational biology and physical sciences. We are not aware of ethical concerns beyond the general ones common to all statistical and data science research.

\bibliography{references}
\bibliographystyle{iclr2027_conference}

\newpage
\appendix

\appendixcontents

\begingroup
\raggedbottom
\section{Notation}
\label{app:notation}

\begin{table}[H]
\centering
\caption{Notation}
\label{tab:notation}
\tablesize
\renewcommand{\arraystretch}{0.87}
\begin{tabularx}{\linewidth}{@{}lX@{}}
\toprule
Symbol & Meaning \\
\midrule
\multicolumn{2}{@{}l}{\emph{Indices}} \\
$i$ & Snapshot, $i=0,\dots,K$ \\
$n$ & Sample within a snapshot, $n=1,\dots,N_i$ \\
$k,l$ & Learned centers, $k,l=1,\dots,P$ \\
$j$ & Time step, $j=0,\dots,M$ \\
\grouprule\addlinespace[2pt]
\multicolumn{2}{@{}l}{\emph{Problem setting}} \\
$\Ptwo$ & Probability measures with finite second moment \\
$\rho_t$ & Population distribution at time $t$ \\
$X_i$ & Observed snapshot at time $t_i$ \\
$N_i$ & Number of observations in snapshot $i$ ($N$ if equal) \\
$d$ & Dimension of the state space \\
$\mathcal T_{\mathrm{obs}}$ & Observation times used for fitting \\
$\mathcal T_{\mathrm{eval}}$ & Times used for evaluation \\
$\hat\rho_t$ & Predicted population distribution \\
$W_p$ & $p$-Wasserstein distance \\
\grouprule\addlinespace[2pt]
\multicolumn{2}{@{}l}{\emph{Dynamics and Clebsch variables}} \\
$u_t$ & General population velocity \\
$v_t$ & Canonical velocity with minimum population kinetic energy \\
$w_t$ & Density-preserving component of the velocity \\
$x_t$ & Particle position at time $t$ \\
$\partial_t f_t(x)$ & Time derivative of a field at fixed position \\
$\mathrm d f_t(x_t)/\mathrm dt$ & Total derivative of a field along a particle trajectory \\
$\mathcal F$ & Free energy of a Wasserstein gradient flow \\
$\mathcal U$ & Potential energy of the population \\
$\mathcal K$ & Kinetic energy of the population \\
$\grad(\delta\mathcal U/\delta\rho_t)$ & Wasserstein gradient of the potential energy \\
$\gamma$ & Linear friction coefficient \\
$\mathcal A$ & Caldirola--Kanai action \\
$\ell_t$ & Material label field \\
$\phi_t$ & Continuity multiplier \\
$\mu_t$ & Clebsch multiplier for label transport \\
\grouprule\addlinespace[2pt]
\multicolumn{2}{@{}l}{\emph{Double-Stitch parameterization}} \\
$\theta$ & Learned parameters \\
$\hat\rho_t^\theta$ & Learned kernel density estimate (KDE) \\
$u_t^\theta$ & Responsibility-weighted mean of the center velocities \\
$x_{t,k}^\theta$ & Position of learned center $k$ at time $t$ \\
$w_k^\theta$ & Time-independent mixture weight of center $k$ \\
$P$ & Number of particles (learned centers in Double-Stitch, simulated particles in WLM) \\
$\phi$ & Smoothing kernel \\
$\sigma$ & Kernel bandwidth \\
$r_{t,k}^\theta(x)$ & Responsibility of center $k$ at position $x$ \\
$\rho_t^P$ & Weighted empirical measure of the centers \\
$V_\theta$ & Learned external potential \\
$W_\theta$ & Learned interaction potential \\
$c_V$ & External-potential coefficient \\
$c_H$ & Entropy coefficient \\
$c_W$ & Interaction coefficient \\
$M$ & Number of time steps (grid intervals or solver steps) \\
$h$ & Time-grid spacing \\
$\Delta_h$ & First-difference operator (central; one-sided at the endpoints) \\
$\Delta_h^2$ & Second-difference operator (central; one-sided at the endpoints) \\
\grouprule\addlinespace[2pt]
\multicolumn{2}{@{}l}{\emph{Objectives and computational cost}} \\
$\mathcal L$ & Training objective; mechanical Lagrangian in the action derivation \\
$\mathcal L_{\mathrm{data}}$ & Data loss \\
$\mathcal L_{\mathrm{res}}$ & Discretized residual loss \\
$\lambda$ & Residual-loss weight \\
$\mathcal D$ & Data divergence \\
$\mathcal R$ & Population Lagrangian residual \\
$C(P)$ & Time of one force evaluation on $P$ particles, including backpropagation \\
$B$ & Number of force evaluations the hardware can run at once \\
\bottomrule
\end{tabularx}
\end{table}

\clearpage
\endgroup

\section{Derivation of the equations of motion}
\label{app:proofs}

    We consider a flow of probability measures $\rho_t\in\Ptwo$ with density and velocity field $u_t$ smooth in space and time, satisfying the continuity equation (\cref{eq:continuity}):
    \begin{equation*}
        \partial_t\rho_t+\dive(\rho_tu_t)=0.
    \end{equation*}

    The classical Lagrangian $\mathcal L=\mathcal K-\mathcal U$ with kinetic energy
    \begin{equation*}
        \mathcal K=\frac12\int\rho_t(x)\norm{u_t(x)}^2\dd x
    \end{equation*}
    and a differentiable potential energy $\mathcal U$ with a smooth first variation describe Hamiltonian dynamics. For canonical velocities, these are Wasserstein Hamiltonian flows (WHF) \citep{chow2020whf}. Total energy is conserved when $\mathcal U$ has no explicit time dependence. The corresponding action is
    \begin{equation*}
        \mathcal A=\int_0^T\big(\mathcal K-\mathcal U\big)\dd t.
    \end{equation*}

    To describe dissipative flows as well, we adopt the Caldirola--Kanai Lagrangian $\mathcal L=\mathrm{e}^{\gamma t}(\mathcal K-\mathcal U)$, where $\gamma\ge0$ is a constant friction coefficient \citep{caldirola1941,kanai1948}. This leads to a damped action
    \begin{equation*}
        \mathcal A=\int_0^T\mathrm{e}^{\gamma t}\big(\mathcal K-\mathcal U\big)\dd t.
    \end{equation*}
    For a time-independent energy, WGF is recovered in the overdamped limit after rescaling time (\cref{app:recovered-flows}).

    Using this action and the continuity constraint, we derive the equations of motion that can be used by a trajectory simulator (for simulation-based methods) or in a residual loss (for simulation-free methods).

\subsection{Lagrange multipliers}

    One common approach to derive the equations of motion is to enforce continuity within the action via a Lagrange multiplier \citep{guan2026wlm}. We write the velocity here as $v_t$ for reasons that will become clear soon.

    Denote the scalar Lagrange multiplier by $\phi_t$ and the spatial inner product by $\langle f,g\rangle=\int f(x)\cdot g(x)\dd x$. Then the constrained action is
    \begin{equation*}
        \mathcal A=\int_0^T\mathrm{e}^{\gamma t}\Big[\mathcal K-\mathcal U+\underbrace{\langle\phi_t,\partial_t\rho_t+\dive(\rho_tv_t)\rangle}_{\text{continuity}}\Big]\dd t.
    \end{equation*}

    Following the principle of stationary action, we set its first variations to zero and integrate by parts:
    \begin{align*}
        \frac{\delta\mathcal A}{\delta\rho_t}&=0
            \implies (\partial_t+v_t\cdot\grad)\phi_t=\frac12\norm{v_t}^2-\fdv{\mathcal U}{\rho_t}-\gamma\phi_t,
            &&\text{(potential evolution)}\\
        \frac{\delta\mathcal A}{\delta\phi_t}&=0
            \implies \partial_t\rho_t+\dive(\rho_tv_t)=0,
            &&\text{(continuity)}\\
        \frac{\delta\mathcal A}{\delta v_t}&=0
            \implies v_t=\grad\phi_t.
            &&\text{(gradient velocity)}
    \end{align*}
    Here $\partial_t$ differentiates a field at a fixed spatial position. Along a moving particle with $\mathrm{d}x_t/\mathrm{d}t=v_t(x_t)$, the ordinary time derivative gives the material derivative by the chain rule,
    \begin{equation*}
        \frac{\mathrm{d}}{\mathrm{d}t}f_t(x_t)
        =\big[\partial_t f_t+(v_t\cdot\grad)f_t\big](x_t).
    \end{equation*} %

    The first equation, together with the third, gives the damped Hamilton--Jacobi equation and recovers the equations of motion in \citet[Equations~(6)--(8)]{guan2026wlm}. In Lagrangian coordinates, with $\frac{\mathrm{d}x_t}{\mathrm{d}t}=v_t(x_t)$, it can be written as
    \begin{equation*}
        \frac{\mathrm{d}^2x_t}{\mathrm{d}t^2}=-\grad\fdv{\mathcal U}{\rho_t}(x_t)-\gamma\frac{\mathrm{d}x_t}{\mathrm{d}t}.
    \end{equation*}

    The second equation is the continuity condition that we wished to impose.

    The third equation restricts the velocity to be a gradient or in other words to be irrotational. This means $v_t$ is the canonical velocity of \citet[Chapter~8]{ambrosio2008} with the smallest kinetic energy among velocity fields producing the same instantaneous change in density (and this is the reason we used $v_t$ instead of $u_t$ to denote it). 

    However, although the rotational component of the velocity ($w_t$ in the Helmholtz decomposition \cref{eq:helmholtz-decomp}) does not affect the density instantaneously, it can still affect the motion of the observed particles and therefore density evolution. We discuss this effect in more detail in \cref{app:velocity-comparison} after introducing the alternative derivation.

\subsection{Clebsch action principle}
\label{app:clebsch-derivation}

    To retain the rotational component of the velocity in the equations of motion, we use the Clebsch action principle \citep{cotter2007clebsch,holm1998}. Here we use the general velocity $u_t$, with particle trajectories satisfying $\mathrm{d}x_t/\mathrm{d}t=u_t(x_t)$.

    Specifically, define a material label field $\ell_t(x)\in\R^d$ that records the initial position of the particle currently at $x$. A particle keeps its identity during the flow, so $\ell_t(x_t)=x_0$ and $\ell_0(x)=x$. This gives the following constraint
    \begin{equation*}
        \partial_t\ell_t+(u_t\cdot\grad)\ell_t=0.
    \end{equation*}

    We adjoin this constraint with a vector multiplier $\rho_t\mu_t$ to obtain the action
    \begin{equation}
        \mathcal A=\int_0^T\mathrm{e}^{\gamma t}\Big[\mathcal K-\mathcal U+\underbrace{\langle\phi_t,\partial_t\rho_t+\dive(\rho_tu_t)\rangle}_{\text{continuity}}-\underbrace{\langle\rho_t\mu_t,(\partial_t+u_t\cdot\grad)\ell_t\rangle}_{\text{label transport}}\Big]\dd t.
        \label{eq:clebsch-action}
    \end{equation}

    \begin{theoremrestate}[Lagrangian equations of motion]
        Let $\rho_t\in\Ptwo$ be a strictly positive density and $u_t$ a velocity field, both smooth in space and time on $[0,T]$ and satisfying continuity. Suppose the particle flow defined by $\mathrm{d}x_t/\mathrm{d}t=u_t(x_t)$ has a smooth inverse. Fix a constant $\gamma\ge0$ and a differentiable energy $\mathcal U$ with a smooth first variation. Assume finite kinetic energy and action, and sufficient integrability to differentiate under the integrals. Variations are smooth and compactly supported in space, with density and particle labels fixed at the temporal endpoints.

        If the Clebsch action in \cref{eq:clebsch-action} is stationary with respect to $\rho_t,u_t$ and smooth labels and multipliers $\ell_t,\phi_t,\mu_t$, then
        \begin{equation*}
            \partial_tu_t+(u_t\cdot\grad)u_t
            =-\gamma u_t-\grad\fdv{\mathcal U}{\rho_t}.
        \end{equation*}
        The velocity $u_t$ need not be a gradient.
    \end{theoremrestate}

    \begin{proof}
    Similarly to the previous derivation, we set the variations to zero for stationary action:
    \begin{equation}
        \begin{aligned}
        \frac{\delta\mathcal A}{\delta\rho_t}&=0
            \implies (\partial_t+u_t\cdot\grad)\phi_t=\frac12\norm{u_t}^2-\fdv{\mathcal U}{\rho_t}-\gamma\phi_t,
            &&\text{(potential evolution)}\\
        \frac{\delta\mathcal A}{\delta\phi_t}&=0
            \implies \partial_t\rho_t+\dive(\rho_tu_t)=0,
            &&\text{(continuity)}\\
        \frac{\delta\mathcal A}{\delta u_t}&=0
            \implies u_t=\grad\phi_t+(\grad\ell_t)^\top\mu_t,
            &&\text{(Clebsch velocity)}\\
        \frac{\delta\mathcal A}{\delta\mu_t}&=0
            \implies (\partial_t+u_t\cdot\grad)\ell_t=0,
            &&\text{(label transport)}\\
        \frac{\delta\mathcal A}{\delta\ell_t}&=0
            \implies (\partial_t+u_t\cdot\grad)\mu_t=-\gamma\mu_t,
            &&\text{(multiplier evolution)}
    \end{aligned}
        \label{eq:clebsch-stationarity}
    \end{equation}

    The first two equations have the same form as before but now with the general velocity $u_t$.

    The third equation adds the label term $(\grad\ell_t)^\top\mu_t$, which permits nonzero vorticity. Here $\grad\phi_t$ is not necessarily the canonical velocity $v_t$ because the label term does not have to preserve density.

    The fourth equation shows that particle keeps its initial label.

    The fifth equation shows that the multiplier is transported with the particles and decays at $\gamma$ rate because of the exponent in the action.

    We use $(\grad\ell_t)_{ab}=\partial_b\ell_{t,a}$, with the same Jacobian convention for $u_t$. The equation of motion then follows from differentiating the Clebsch velocity along a particle. For a scalar field $f_t$, the chain rule gives
    \begin{equation*}
        \frac{\mathrm{d}}{\mathrm{d}t}\big[\grad f_t(x_t)\big]
        =\Big[\grad\big(\partial_t f_t+(u_t\cdot\grad)f_t\big)
        -(\grad u_t)^\top\grad f_t\Big](x_t).
    \end{equation*}
    We apply this to $\phi_t$ and each label component and substitute the stationarity conditions:
    \begin{align}
        \frac{\mathrm{d}}{\mathrm{d}t}u_t(x_t)
        &=\Big[\grad\left(\frac12\norm{u_t}^2-\fdv{\mathcal U}{\rho_t}\right)
            -(\grad u_t)^\top u_t-\gamma u_t\Big](x_t)\notag\\
        &=-\grad\fdv{\mathcal U}{\rho_t}(x_t)-\gamma u_t(x_t),
        \label{eq:clebsch-law}
    \end{align}
    where we use $(\grad u_t)^\top u_t=\frac{1}{2}\grad \left(\norm{u_t}^2\right)$. In Lagrangian coordinates, with $\frac{\mathrm{d}x_t}{\mathrm{d}t}=u_t(x_t)$, this becomes
    \begin{equation*}
        \frac{\mathrm{d}^2x_t}{\mathrm{d}t^2}=-\grad\fdv{\mathcal U}{\rho_t}(x_t)-\gamma\frac{\mathrm{d}x_t}{\mathrm{d}t}.
    \end{equation*}
    This is the same particle equation as before, now allowing a general initial velocity.

    \end{proof}

\subsection{Recovered flows}
\label{app:recovered-flows}

    The dynamical law in \cref{eq:clebsch-law} describes damped flow with a general initial velocity. We now show how it recovers the canonical dynamics of WLM, conservative dynamics, Wasserstein Hamiltonian flow, and Wasserstein gradient flow.

\paragraph{Canonical dynamics (WLM)}
    If the initial velocity is a gradient: $u_0 = v_0 = \grad\phi_0$. Choosing $\mu_0=0$ and using the multiplier evolution in \cref{eq:clebsch-stationarity}, we get $\mu_t=0$ throughout the smooth flow. The velocity then remains canonical: $u_t=v_t=\grad\phi_t$. The dynamical law becomes
    \begin{equation*}
        \partial_t\phi_t+\frac12\norm{\grad\phi_t}^2+\fdv{\mathcal U}{\rho_t}+\gamma\phi_t=0.
    \end{equation*}
    Together with continuity, this recovers the Hamiltonian equations of motion of \citet[Theorem~2.2]{guan2026wlm} and proves \cref{cor:caldirola-kanai}.

\paragraph{Conservative dynamics}
    With $\gamma=0$ and no explicit time dependence in $\mathcal U$ (otherwise its time derivative contributes to the energy balance as well), the flow conserves total energy, including for general velocities. Differentiating the potential energy along the flow and using continuity gives
    \begin{equation*}
        \frac{\dd}{\dd t}\mathcal U=\int\rho_tu_t\cdot\grad\fdv{\mathcal U}{\rho_t}\dd x.
    \end{equation*}
    Similarly, differentiating the full energy and substituting the dynamical law gives
    \begin{equation*}
        \frac{\dd}{\dd t}\big(\mathcal K+\mathcal U\big)
        =\int\rho_tu_t\cdot\left(\partial_tu_t+(u_t\cdot\grad)u_t+\grad\fdv{\mathcal U}{\rho_t}\right)\dd x
        =-\gamma\int\norm{u_t}^2\rho_t\dd x.
    \end{equation*}
    Therefore the total energy is conserved for $\gamma=0$ and dissipated for $\gamma>0$.

\paragraph{Wasserstein Hamiltonian flow}
    Setting $\gamma=0$ in the canonical dynamics gives Wasserstein Hamiltonian flow \citep{ambrosiogangbo2008,chow2020whf}. Here $\rho_t$ and $\phi_t$ satisfy continuity and the undamped Hamilton--Jacobi equation. For a time-independent potential energy, the Hamiltonian (total energy) $\mathcal K+\mathcal U$ is conserved (see above).

\paragraph{Wasserstein gradient flow}
    For a time-independent energy $\mathcal U=\F$, the particle form of \cref{eq:clebsch-law} has the same form as the one used by \citet[Proposition~2.4]{guan2026wlm} to recover WGF in the overdamped limit. Explicitly, in slow time $\tau=t/\gamma$, writing $\widetilde x_\tau=x_{\gamma\tau}$ and $\widetilde\rho_\tau=\rho_{\gamma\tau}$ gives
    \begin{equation*}
        \frac{1}{\gamma^{2}}\frac{\mathrm{d}^2\widetilde x_\tau}{\mathrm{d}\tau^2}+\frac{\mathrm{d}\widetilde x_\tau}{\mathrm{d}\tau}
        =-\grad\fdv{\F}{\widetilde\rho_\tau}(\widetilde x_\tau).
    \end{equation*}
    Formally dropping the inertial term as $\gamma\to\infty$ yields $\mathrm{d}\widetilde x_\tau/\mathrm{d}\tau=-\grad\fdv{\F}{\widetilde\rho_\tau}(\widetilde x_\tau)$ and hence the WGF law \cref{eq:wgf}.

\subsection{General and canonical velocities}
\label{app:velocity-comparison}

    There may be multiple velocity fields that produce the same instantaneous change in density. The canonical velocity $v_t$ has the minimum kinetic energy among such fields \citep{otto2001,ambrosio2008}. 
    
    However, the rotational part of the velocity field can still lead to later changes in the density that the canonical velocity would not be able to capture. For example, consider an elliptical (anisotropic two-dimensional) Gaussian distribution rotating around its center under a harmonic force at the rotation frequency (that would keep the particles together). Its general velocity would simply rotate the distribution. Its canonical velocity would produce the same initial change in density, but canonical velocity would stretch the ellipse along one diagonal and compress it along the other over longer time.

    \paragraph{Relation to WLM}
    \label{app:wlm-velocity-scope}
    \citet[Theorem~2.2 and Appendix~A]{guan2026wlm} derive the dynamical law for canonical velocities ($v_0=\grad\phi_0$ in our notation). Their simulation-based algorithm can use the initial velocities directly, without imposing an additional gradient condition \citep[Algorithms~1 and~2]{guan2026wlm}. Therefore, if the initial velocity is canonical, WLM would preserve it for later time steps. However, if the velocity is general, the overall trajectory would not be canonical anymore. The Clebsch derivation covers smooth general initial velocities, so it allows for the same particle law but with the velocities that can include circulation.

\section{Derivation of Double-Stitch}
\label{app:discretization}

\subsection{Velocity and acceleration}
\label{app:kde-velocity}

    Consider time-independent weights and a fixed kernel $\phi$. We differentiate the KDE path \cref{eq:kde-path} and use $\partial_t\phi(x-x_{t,k})=-\dot x_{t,k}\!\cdot\grad\phi(x-x_{t,k})$
    \begin{equation*}
        \partial_t\hat\rho_t(x)
        =-\sum_{k=1}^P w_k\,\dot x_{t,k}\!\cdot\grad\phi(x-x_{t,k})
        =-\dive\Big(\sum_{k=1}^P w_k\,\phi(x-x_{t,k})\,\dot x_{t,k}\Big)
        =-\dive\big(\hat\rho_t(x)\,u_t(x)\big),
    \end{equation*}
    so the KDE path satisfies the continuity equation with the velocity field
    \begin{equation*}
        u_t(x)=\sum_{k=1}^P r_{t,k}(x)\,\dot x_{t,k},
        \qquad
        r_{t,k}(x)=\frac{w_k\,\phi(x-x_{t,k})}{\hat\rho_t(x)},
    \end{equation*}
    which is \cref{eq:kde-velocity} (the responsibility-weighted mean of the center velocities). Then its material derivative is
    \begin{align*}
        \partial_tu_t(x)+(u_t\!\cdot\grad)u_t(x)
        &=\sum_{k=1}^P r_{t,k}(x)\,\ddot x_{t,k}
          -\frac{1}{\hat\rho_t(x)}\dive\!\big(\hat\rho_t(x)\,\Sigma_t(x)\big),
        \\
        \Sigma_t(x)&=\sum_{k=1}^P r_{t,k}(x)\,\big(\dot x_{t,k}-u_t(x)\big)\big(\dot x_{t,k}-u_t(x)\big)^{\!\top},
    \end{align*}
    where $\Sigma_t$ is the dispersion of the center velocities about their local mean.

\subsection{Energy variation and Wasserstein gradient}
\label{app:wasserstein-gradient}
    
    Consider a symmetric interaction kernel $W_\theta(z)=W_\theta(-z)$ (e.g., a radial interaction). We compute the first variation of the energy \cref{eq:stitching-functional}:
    \begin{equation*}
        \frac{\delta\mathcal U_\theta}{\delta\rho}(x) =c_VV_\theta(x)+c_H\bigl(\log\rho(x)+1\bigr) +c_W(W_\theta*\rho)(x),
    \end{equation*}
    For an asymmetric kernel, this energy depends only on its even part, so we replace $W_\theta$ in the first variation with $\frac{1}{2}(W_\theta(z)+W_\theta(-z))$. For the symmetric case, the analytic Wasserstein gradient is
    \begin{equation}
        \nabla_x\frac{\delta\mathcal U_\theta}{\delta\rho}(x) =c_V\nabla V_\theta(x)+c_H\nabla\log\rho(x) +c_W\nabla(W_\theta*\rho)(x).
        \label{eq:stitching-functional-gradient}
    \end{equation}
    We then evaluate the entropy term on the KDE path: $\grad\log\hat\rho_t$. Approximating the interaction with the empirical mass measure $\rho_t^P=\sum_k w_k\delta_{x_{t,k}}$ we get the force used in the residual
    \begin{equation}
        \grad\fdv{\mathcal U_\theta}{\rho}\big[\hat\rho_t\big](x)
        \approx c_V\grad V_\theta(x)+c_H\grad\log\hat\rho_t(x)
        +c_W\sum_{l=1}^P w_l\grad W_\theta(x-x_{t,l}).
        \label{eq:kde-wasserstein-gradient}
    \end{equation}

\subsection{Lagrangian residual and empirical approximation}
\label{app:residual-theorem}

    We define the Lagrangian residual as the integral over time and space of the law deviation from zero (see \cref{eq:general-motion-law}):
    \begin{equation}
        \mathcal R[\rho,u]=\int_0^T\!\int
        \norm*{\partial_tu_t+(u_t\cdot\grad)u_t+\gamma u_t+\grad\fdv{\mathcal U}{\rho_t}}^2
        \rho_t(x)\dd x\dd t.
        \label{eq:population-residual}
    \end{equation}
    Under the assumptions in \cref{app:clebsch-derivation}, a trajectory satisfying the equation of motion has zero residual. Conversely, a zero residual means that the nonnegative integrand vanishes almost everywhere. Then, using smoothness and strict positivity of $\rho_t$ we obtain satisfaction of the equation of motion at every point.
    
    Finally, continuity ensures that particles starting from $\rho_0$ have density $\rho_t$ at time $t$. Since $\mathrm{d}u_t(x_t)/\mathrm{d}t=\mathrm{d}^2x_t/\mathrm{d}t^2$, we can write the same residual along particle trajectories as
    \begin{equation*}
        \mathcal R[\rho,u]
        =\mathbb E_{x_0\sim\rho_0}\!\left[\int_0^T
        \norm*{\frac{\mathrm{d}^2x_t}{\mathrm{d}t^2}+\gamma\frac{\mathrm{d}x_t}{\mathrm{d}t}
        +\grad\fdv{\mathcal U}{\rho_t}(x_t)}^2\dd t\right],
    \end{equation*}
    which gives \cref{eq:residual-pullback}.
        
    We then approximate the integral using the particle representation \cref{eq:residual-pullback}. Specifically, each path has velocity $\dot x_{t,k}$ and acceleration $\ddot x_{t,k}$ so its mass measure satisfies
    \begin{equation*}
        \int g\,\dd\rho_t^P=\sum_{k=1}^P w_k\,g(x_{t,k}).
    \end{equation*}
    We can then evaluate density-dependent forces using the KDE at a finite bandwidth and approximate the interaction using the centers as in \cref{eq:kde-wasserstein-gradient}. The residual is then approximated by
    \begin{equation*}
        \mathcal R
        \approx\int_0^T\sum_{k=1}^P w_k\,
          \Big\|\,\ddot x_{t,k}+\gamma\,\dot x_{t,k}
           +\grad\fdv{\mathcal U}{\rho}\big[\hat\rho_t\big](x_{t,k})\Big\|^2\dd t.
    \end{equation*}
    
    We note that this is a particle approximation to the continuum construction, not an exact evaluation of \cref{eq:population-residual} for the KDE mean velocity.

\subsection{Time discretization}
\label{app:time-discretization}
    
    We perform finite-difference parameterization on the time grid $t_0<\dots<t_M$ with steps $h_j=t_{j+1}-t_j$. The center velocity and acceleration can be replaced by the three-point central differences at the interior nodes $j=1,\dots,M-1$,
    \begin{align*}
        \Delta_h x_{t_j,k}&=\frac{x_{t_{j+1},k}-x_{t_{j-1},k}}{h_{j-1}+h_j},
        \\
        \Delta_h^2 x_{t_j,k}&=\frac{2}{h_{j-1}+h_j}\left(\frac{x_{t_{j+1},k}-x_{t_j,k}}{h_j}-\frac{x_{t_j,k}-x_{t_{j-1},k}}{h_{j-1}}\right),
    \end{align*}
    which reduce on a uniform grid $h_j\equiv h$ to
    \begin{equation*}
        \Delta_h x_{t_j,k}=\frac{x_{t_{j+1},k}-x_{t_{j-1},k}}{2h},\qquad
        \Delta_h^2 x_{t_j,k}=\frac{x_{t_{j+1},k}-2x_{t_j,k}+x_{t_{j-1},k}}{h^2}.
    \end{equation*}
    The time integral is evaluated by the trapezoidal rule at all nodes, setting $h_{-1}=h_M=0$,
    \begin{equation*}
        \mathcal R
        \approx\sum_{j=0}^{M}\tfrac12(h_{j-1}+h_j)\sum_{k=1}^P w_k\,
          \Big\|\,\Delta_h^2x_{t_j,k}+\gamma\,\Delta_h x_{t_j,k}
           +\grad\fdv{\mathcal U}{\rho}\big[\hat\rho_{t_j}\big](x_{t_j,k})\Big\|^2,
    \end{equation*}
    which is the loss \cref{eq:particle-second-order-residual}.
    The two endpoint nodes use a one-sided stencil of the same order.

\begingroup
\raggedbottom
\subsection{Pseudocode}
\label{app:stitching-algorithm}
    
    We present the pseudocode for WLM in \cref{alg:wlm} and Double-Stitch in \cref{alg:double-stitch}. The former is adapted from the original paper to our notation \citep{guan2026wlm}.

    \begin{algorithm}[H]
  \caption{\textsc{WLM}}
  \label{alg:wlm}
  \small
  \noindent\hspace*{-1em}%
  \begin{minipage}{\dimexpr\linewidth+1em\relax}
  \begin{algorithmic}
    \Statex \textbf{Input}
    \AlgSubitem{\textbf{Data:} snapshots $\{X_i\}_{t_i\in\mathcal T_{\mathrm{obs}}}$, $\mathcal T_{\mathrm{obs}}=\{t_0<\dots<t_K\}$, represented by samples $\{x_i^{(n)}\}_{n=1}^{N}$; initial velocities $\{v_0^{(n)}\}_{n=1}^{N}$ at $t_0$}
    \AlgSubitem{\textbf{Model hyperparameters:} divergence $\mathcal D$, friction $\gamma\ge0$, and solver step size $\Delta t$}
    \AlgSubitem{\textbf{Training budget:} number of iterations $n_{\mathrm{iters}}$}

    \Statex \vspace{0.2em}
    \Statex \textbf{Initialization}
    \AlgSubitem{\textbf{Energy:} network parameters of $\mathcal U_\theta:\mathbb R^{P\times d}\to\mathbb R$}
    \AlgSubitem{\textbf{Friction:} $\gamma$ fixed, or initialized when learned}

    \Statex \vspace{0.25em}
    \Statex \textbf{Training}
    \TrainFor{$\mathrm{iter}=1,\ldots,n_{\mathrm{iters}}$}
      \State \hspace{\algorithmicindent}\textbf{Rollout}
      \State \hspace{\algorithmicindent}$H\sim\operatorname{Unif}\{1,\ldots,K\}$ \Comment{random horizon}
      \State \hspace{\algorithmicindent}$\tilde x_0^{(n)}\gets x_0^{(n)},\qquad \tilde v_0^{(n)}\gets v_0^{(n)}$ \Comment{$P$ particles from the observed initial cloud}
      \State \hspace{\algorithmicindent}\textbf{for} $j=0,\ldots,\lceil (t_H-t_0)/\Delta t\rceil-1$ \textbf{do}
      \State \hspace{\algorithmicindent}\hspace{\algorithmicindent}$a_j^{(n)}\gets-\nabla_{\tilde x^{(n)}}\mathcal U_\theta\bigl(\tilde x_j^{(1)},\ldots,\tilde x_j^{(P)}\bigr)$ \Comment{for all $n$}
      \State \hspace{\algorithmicindent}\hspace{\algorithmicindent}$\tilde x_{j+1}^{(n)},\tilde v_{j+1}^{(n)}\gets\operatorname{Leapfrog}\bigl(\tilde x_j^{(n)},\tilde v_j^{(n)},a_j^{(n)},\gamma,\Delta t\bigr)$ \Comment{damped leapfrog of \cref{eq:wlm}}
      \State \hspace{\algorithmicindent}\textbf{end for}
      \State \hspace{\algorithmicindent}$\hat X_i\gets\{\tilde x_{j_i}^{(n)}\}_{n=1}^{P}$ at the steps $j_i$ reaching $t_i$, $i=1,\ldots,H$
      \Statex \vspace{0.1em}
      \State \hspace{\algorithmicindent}\textbf{Objective and update}
      \State \hspace{\algorithmicindent}$\displaystyle \mathcal L_{\mathrm{WLM}}(\theta,\gamma)\gets\frac1H\sum_{i=1}^{H}\mathcal D(\hat X_i,X_i)$ \Comment{backpropagated through the rollout}
      \State \hspace{\algorithmicindent}$\displaystyle \operatorname{OptStep}_{\mathcal L_{\mathrm{WLM}}}\!\bigl(\mathcal U_\theta,\ \gamma\text{ if learned}\bigr)$
    \EndTrainFor
    \Statex \vspace{0.2em}
    \Statex \textbf{Output}
    \AlgSubitem{\textbf{Learned energy and friction:} $\mathcal U_\theta$, $\gamma$}
    \AlgSubitem{\textbf{Learned law:} \cref{eq:wlm}, simulated by the same integrator from any observed cloud and velocity}
  \end{algorithmic}
  \end{minipage}
\end{algorithm}

    \begin{algorithm}[H]
  \caption{\textsc{Double-Stitch}}
  \label{alg:double-stitch}
  \small
  \noindent\hspace*{-1em}%
  \begin{minipage}{\dimexpr\linewidth+1em\relax}
  \begin{algorithmic}
    \Statex \textbf{Input}
    \AlgSubitem{\textbf{Data:} snapshots $\{X_i\}_{t_i\in\mathcal T_{\mathrm{obs}}}$, represented by samples; optionally, paired initial positions and velocities $\{(x_0^{(n)},v_0^{(n)})\}_{n=1}^{N_0}$}
    \AlgSubitem{\textbf{Model hyperparameters:} number of learned particles $P$, time grid $t_0<\dots<t_M$ with spacing $h$, initial KDE bandwidth $\sigma_0$, initial friction $\gamma_0$, residual-loss weight $\lambda$, and initial-velocity weight $\lambda_{v_0}$ ($0$ without initial velocities)}
    \AlgSubitem{\textbf{Training budget:} number of epochs $n_{\mathrm{epochs}}$}

    \Statex \vspace{0.2em}
    \Statex \textbf{Initialization}
    \AlgSubitem{\textbf{Particle positions:} $\{x_{t_j,k}^\theta\}_{j,k}\gets \operatorname{InitPath}(\{X_i\}_{t_i\in\mathcal T_{\mathrm{obs}}},P,M)$}
    \AlgSubitem{\textbf{KDE:} bandwidth $\sigma^\theta\gets\sigma_0$ and uniform mixture weights $w_{1:P}^\theta\gets1/P$}
    \AlgSubitem{\textbf{Potential:} network parameters $V_\theta,W_\theta$ and functional coefficients $c_V,c_H,c_W$}
    \AlgSubitem{\textbf{Friction:} $\gamma\gets\gamma_0$}

    \Statex \vspace{0.25em}
    \Statex \textbf{Training}
    \TrainFor{$\mathrm{epoch}=1,\ldots,n_{\mathrm{epochs}}$}
      \State \hspace{\algorithmicindent}\textbf{Density and force}
      \State \hspace{\algorithmicindent}$\displaystyle \hat\rho_{t_j}^\theta(x)\gets\sum_{k=1}^{P}w_k^\theta\phi_{\sigma^\theta}(x-x_{t_j,k}^\theta)$ \Comment{\cref{eq:kde-path}}
      \State \hspace{\algorithmicindent}$\displaystyle g_{j,k}^{\theta}\gets c_V\nabla V_\theta(x_{t_j,k}^\theta)+c_H\nabla\log\hat\rho_{t_j}^\theta(x_{t_j,k}^\theta)+c_W\sum_{l=1}^{P}w_l^\theta\nabla W_\theta(x_{t_j,k}^\theta-x_{t_j,l}^\theta)$ \Comment{\cref{eq:kde-wasserstein-gradient}}
      \Statex \vspace{0.1em}
      \State \hspace{\algorithmicindent}\textbf{Acceleration residual}
      \State \hspace{\algorithmicindent}$\displaystyle \Delta_h x_{t_j,k}^\theta\gets\frac{x_{t_{j+1},k}^\theta-x_{t_{j-1},k}^\theta}{2h},\qquad \Delta_h^2 x_{t_j,k}^\theta\gets\frac{x_{t_{j+1},k}^\theta-2x_{t_j,k}^\theta+x_{t_{j-1},k}^\theta}{h^2}$ \Comment{$0<j<M$}
      \State \hspace{\algorithmicindent}$\Delta_h x_{t_j,k}^\theta,\ \Delta_h^2 x_{t_j,k}^\theta$ at $j=0,M$: one-sided second-order stencils
      \State \hspace{\algorithmicindent}$\displaystyle \mathcal L_{\mathrm{res}}(\theta,\gamma)\gets h\sum_{j=0}^{M}\sum_{k=1}^{P}w_k^\theta\left\lVert \Delta_h^2 x_{t_j,k}^\theta+\gamma\,\Delta_h x_{t_j,k}^\theta+g_{j,k}^{\theta}\right\rVert_2^2$ \Comment{endpoint terms halved}
      \Statex \vspace{0.1em}
      \State \hspace{\algorithmicindent}\textbf{Initial velocity (optional)}
      \State \hspace{\algorithmicindent}$\displaystyle \mathcal L_{v_0}(\theta)\gets\frac{1}{N_0}\sum_{n=1}^{N_0}\Bigl\lVert\sum_{k=1}^{P}r_{t_0,k}^{\theta}(x_0^{(n)})\,\frac{x_{t_1,k}^{\theta}-x_{t_0,k}^{\theta}}{h}-v_0^{(n)}\Bigr\rVert_2^2$ \Comment{\cref{eq:initial-velocity-loss}}
      \Statex \vspace{0.1em}
      \State \hspace{\algorithmicindent}\textbf{Objective and update}
      \State \hspace{\algorithmicindent}$\displaystyle \mathcal L_{\mathrm{data}}(\theta)\gets\sum_{t_i\in\mathcal T_{\mathrm{obs}}}\mathcal D(\hat\rho_{t_i}^\theta,X_i)$
      \State \hspace{\algorithmicindent}$\displaystyle \mathcal L(\theta,\gamma)\gets\mathcal L_{\mathrm{data}}(\theta)+\lambda\mathcal L_{\mathrm{res}}(\theta,\gamma)+\lambda_{v_0}\mathcal L_{v_0}(\theta)$ \Comment{\cref{eq:stitching-losses}}
      \State \hspace{\algorithmicindent}$\displaystyle \operatorname{OptStep}_{\mathcal L}\!\bigl(\{x_{t_j,k}^\theta\}_{j,k},\ \sigma^\theta,\ w_{1:P}^\theta,\ V_\theta,\ W_\theta,\ c_V,\ c_H,\ c_W,\ \gamma\bigr)$
    \EndTrainFor
    \Statex \vspace{0.2em}
    \Statex \textbf{Output}
    \AlgSubitem{\textbf{Learned particle positions:} $\{x_{t_j,k}^\theta\}_{j,k}$}
    \AlgSubitem{\textbf{Learned density:} induced curve $\hat\rho^\theta=(\hat\rho_t^\theta)_{t\in[0,T]}$; \cref{eq:kde-path}}
    \AlgSubitem{\textbf{Learned velocity:} induced field $u^\theta=(u_t^\theta)_{t\in[0,T]}$; \cref{eq:kde-velocity}}
    \AlgSubitem{\textbf{Learned potentials and friction:} energy $\mathcal U_\theta$ induced by external $V_\theta$ and interaction $W_\theta$; \cref{eq:stitching-functional}; friction $\gamma$}
  \end{algorithmic}
  \end{minipage}
\end{algorithm}

\subsection{Computational complexity}
\label{app:proof-complexity}
        
    \begin{propositionrestate}[Training cost]
        Consider the training cost of one Double-Stitch and WLM training iteration on the same task. Suppose that the number of time steps $M$, number of particles $P$, and the data loss evaluation time are the same. Assume that $C(P)$ is the time it takes to make a force evaluation ($P$ particles, including backpropagation), and suppose the hardware can run $B$ force evaluations in parallel. Then, apart from the data loss, one Double-Stitch iteration takes $O(\lceil M/B\rceil\,C(P))$ time while WLM will take $O(M\,C(P))$.
    \end{propositionrestate}
    \begin{proof}
        Since we assume that the parameters are identical, both methods will arrive at a computational grid of the same size. Therefore if we count force evaluations, both methods will require $O(M)$ evaluations.
        
        However WLM has to unroll a simulated trajectory while Double-Stitch stores it as the learned particles. The unroll has to happen sequentially while Double-Stitch is able to compute the residual in parallel among many nodes (see \cref{eq:particle-second-order-residual}).
        
        WLM evaluates one force for each of $M$ steps and the steps have to be sequential (see \cref{alg:wlm}), therefore the complexity is $O(M\,C(P))$.
        
        Double-Stitch on the other hand does not require a node to use another's output (it only reads the positions of its neighbors), meaning we can compute forces for $M+1$ terms in batches of $B$ terms and add up their gradients, since the loss is their sum, making its complexity $\lceil (M+1)/B\rceil\,C(P)$ time.

        For $B=1$ both methods take $O(M\,C(P))$ time, and if both have the same $C(P)$, Double-Stitch is up to $\min(B,M)$ times faster. Note that this compares a single update, not the number of updates needed to fit the data. In reality Double-Stitch also uses KL, which is generally cheaper to compute than Sinkhorn.
    \end{proof}
\endgroup

\section{Experimental details}
\label{app:technical-details}
\raggedbottom

\subsection{Data}
\label{app:data-details}

    \paragraph{Synthetic} The demonstration in \cref{fig:gaussian-mixture,fig:synthetic-density-comparison} uses three flows: conservative (Hamiltonian), damped (Lagrangian), and gradient (WGF). 
    
    The conservative flow is a harmonic process:
    \begin{equation*}
        \ddot x_t=-\gamma\dot x_t-\omega_0^2x_t,\quad \omega_0=2\pi/2.7, \quad \gamma=0.
    \end{equation*}
    The damped flow is the same process but with $\gamma = 1$.

    The gradient flow is an Ornstein--Uhlenbeck process:
    \begin{equation*}
        \mathrm dx_t=-x_t\dd t+\sqrt2\dd B_t, \quad \rho_0 = \frac{1}{2}\mathcal N(-2.5,0.30^2)+\frac{1}{2}\mathcal N(2.5,0.65^2), \quad \rho_{t \rightarrow \infty} = \mathcal N(0,1),
    \end{equation*}
    where $B_t$ is a standard Brownian motion.
    The solutions for all three flows are analytic: the particle dynamics are linear in position (plus velocity for the harmonic flow) so each Gaussian component of the initial distribution stays Gaussian \citep[Ch.~6]{sarkka2019}.
    
    For training, we sample $100$ particles at each of five equally spaced times on $[0.15,1.5]$. We use even snapshots ($0$, $2$, $4$) for training and odd ones ($1$, $3$) for testing (similar to L2O on the EB dataset).

    \paragraph{Single-cell data} The embryoid-body (EB) dataset follows human embryonic stem-cell differentiation \citep{moon2019}. We use the processed data from TrajectoryNet \citep{tong2020}, also used by WLM and Stitching. It contains $16819$ cells in five time bins, retains the first five principal components and standardizes each dimension using its mean and std across all cells. LPO uses a fixed random 30\% of cells for validation ($11773$ for training and $5046$ for testing) with the $W_1$ metric averaged over the four noninitial times. L1O withholds each intermediate snapshot and fits a separate model on each with the $W_1$ metric averaged over the model performances. L2O trains on snapshots $0,2,4$ and evaluates $W_2$ at snapshots $1,3$ (we also report the average).

    We visualize the EB data in \cref{fig:eb-data-snapshots} using a common two-dimensional UMAP \citep{mcinnes2018}.

    \begin{figure}[H]
    \centering
    \includegraphics[width=\linewidth]{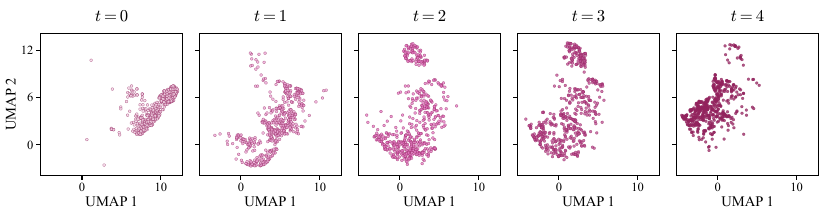}
    \caption{\textbf{EB snapshots over time.} Shared UMAP.}
    \label{fig:eb-data-snapshots}
\end{figure}

    \paragraph{Ocean vortex data} HYCOM ocean vortex datasets from the Gulf of Mexico \citep{guan2026wlm} contain 9 snapshots of $111$ particles for the small vortex and 11 snapshots of $400$ particles for the large vortex. For both, we use WLM's preprocessing. Small vortex offers an interpolation task: even snapshots are used for training and odd ones are held out for testing. The large-vortex dataset offers an extrapolation (or a forecasting task): snapshots $0$ to $9$ are used for training, and snapshot $10$ used for testing. Snapshot times are one time unit apart in all cases, snapshots are considered unpaired.

    We visualize the two ocean vortices in \cref{fig:oceans-data-snapshots}.

    \begin{figure}[H]
    \centering
    \includegraphics[width=\linewidth]{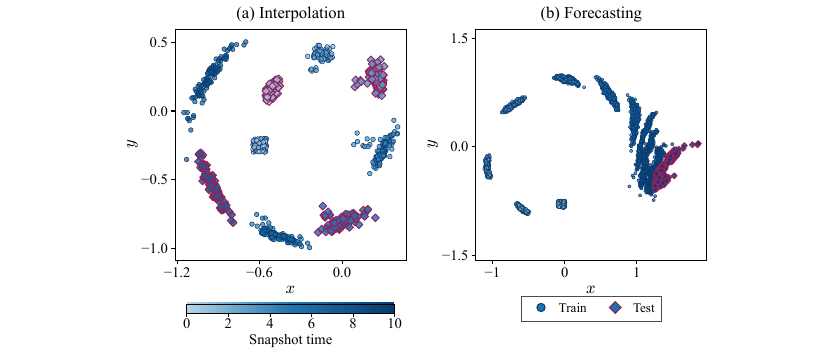}
    \caption{\textbf{Ocean vortex snapshots.} Color indicates time. Circular points are used for training, diamonds for testing.}
    \label{fig:oceans-data-snapshots}
\end{figure}

\subsection{Training and evaluation}
\label{app:training-details}

    We use settings similar to Stitching and WLM and list the main ones below; the configuration files in our code give the complete settings.

    \paragraph{Hardware and uncertainty} We run all timing experiments on NVIDIA RTX 4090 GPUs. We report accuracy as mean \textpm{} std over five seeds for EB and three seeds for oceans. Ablations, sensitivity studies and training time comparisons use three seeds. For published results, we retain the reported means and standard deviations.

    \paragraph{Double-Stitch} On EB and oceans, we parametrize the potential by an MLP with two hidden layers of $64$ SiLU units. We initialize the trajectories by OT-path, the learnable bandwidth at the coarsest scale, and the learnable friction at $\gamma_0=1$. On EB, we largely follow the Stitching settings of \citet[Appendix~G]{heinonen2026}: $50$ time nodes, $10000$ updates, learned mixture weights and batches of $256$, with only the cell coordinates as observations. We use $P=1000$ centers. On oceans, we use $P=100$ centers and $49$ time nodes for the small vortex, and $P=400$ centers and $51$ time nodes for the large vortex. The latter grid extends to the held-out time. Both ocean tasks use full batches and fixed uniform mixture weights.

    We train with Adam for $10000$ updates on EB and $40000$ on oceans. The learning rate starts at $0.005$ and follows cosine decay; the potential network has decoupled weight decay of $1.0$. We set the residual-loss weight to $\lambda=1$. The coefficients $c_V$ and $c_H$ are learned from initial values of $1$ and $0.03$, respectively, while $c_W=0$. We treat the mixture weights as constants inside the residual, so this loss does not update them. On oceans, we also use the initial-velocity constraint in \cref{eq:initial-velocity-loss}. Its weight $\lambda_{v_0}$ is the inverse empirical mean squared initial speed. We evaluate the final checkpoint and use $50$ damped-Verlet substeps per snapshot interval when integrating the learned dynamics.

    \paragraph{Stitching}
    \label{app:reproduction-audit}
    We report two settings for our Stitching reruns. On EB, \emph{original settings} follow Appendix G of \citet{heinonen2026}: $P=100$ centers, $50$ time nodes, a potential with two hidden layers of $64$ SiLU units, $10000$ Adam updates and a learnable bandwidth initialized by Silverman's rule. We scale time to $[0,1]$, use a $70/30$ cell split, learn the mixture weights and train with batches of $256$. We evaluate the learned density directly at the requested time. Our trajectory initialization differs from Appendix G: we use endpoint OT, while the original method couples consecutive snapshots. On oceans, \emph{original settings} use our first-order Stitching configuration with a budget of $40000$ updates and early stopping (patience $2000$, minimum training-loss improvement $10^{-4}$). We evaluate the checkpoint with the lowest training loss.

    For \emph{main settings}, we use the Double-Stitch setup with the velocity residual in \cref{eq:wgf}, without the initial-velocity constraint or early stopping. On EB, we use $P=400$ centers. We report static and time-dependent potentials separately.

    \paragraph{WLM} We use the released architectures and solver, with $10000$ training updates on EB and $50000$ on oceans. On EB, we use the released RNA velocities and a fixed Sinkhorn blur of $1.25$. On oceans, we distinguish two settings by how the blur is estimated. The \emph{main settings}, used in \cref{tab:oceans-comparison}, estimate it from the training snapshots only. The \emph{original settings} estimate it from all snapshots, as in the released code. Both settings use the same position clouds and training/test split, fit only the training snapshots and report the final checkpoint over all three seeds. WLM's attention potential takes time as an input for ocean interpolation, but not for EB or ocean forecasting. We distinguish published results from our reruns in \cref{app:eb-results,app:oceans-results}.

    \paragraph{Initialization sensitivity} We compare three trajectory initializations. Endpoint OT matches the first and last snapshots by optimal transport and linearly interpolates between them. OT-path instead chains optimal-transport matchings between consecutive snapshots. McCann-path approximates each snapshot by a Gaussian with the same moments and follows the Bures--Wasserstein geodesics between them \citep{mccann1997}.

    We also compare two bandwidth initializations. Silverman uses Silverman's rule \citep{silverman1986}, $1.06\,\hat\sigma_i N^{-1/(d+4)}$ in each dimension $i$, computed on the first training snapshot. Coarsest uses $\sqrt{\lambda_{\max}}$, the standard deviation along the principal axis of the pooled training snapshots, the coarsest scale at which a Gaussian mixture still resolves structure \citep{rose1990}. We first vary the bandwidth initialization with OT-path fixed, then vary the trajectory initialization with Silverman bandwidth fixed. For each seed, we keep the network initialization the same across these comparisons.

    \paragraph{Prediction} We use the three prediction modes described in \cref{sec:experiments}. Direct prediction evaluates the learned KDE $\hat\rho_t^\theta$. Pushforward integrates the learned dynamics from the latest observed snapshot before the test time. Rollout integrates from the initial snapshot without restarting at later observations. For second-order dynamics, pushforward also uses the observed velocities at its starting snapshot, while rollout uses only the global initial velocity. On EB, we use direct prediction for Stitching and Double-Stitch, and pushforward for WLM. On oceans, we report each protocol separately.

    \paragraph{Training time} Double-Stitch's training time in \cref{tab:actual-cost} includes initialization, compilation and training callbacks. WLM's covers its native training loop, including monitoring but not model setup. Neither measurement includes the final evaluation. Double-Stitch uses seeds $1$--$3$ for EB LPO/L1O and $0$--$2$ for the other tasks. WLM uses original settings, with seeds $2$--$4$ for EB LPO/L1O and $0$--$2$ for EB L2O and both ocean tasks. L1O sums the training times of the three holdout fits within each seed.

    For the particle-count comparison in \cref{tab:scaling}, we measure only optimizer execution time, synchronized on the GPU. We use seeds $0$--$2$ for all five tasks.

    \paragraph{Synthetic demonstration} For the six fits in \cref{fig:synthetic-density-comparison}, we use seed $0$, $P=50$ centers, $50$ time nodes and $5000$ Adam updates. The learning rate starts at $0.005$ and follows cosine decay, without weight decay. Both methods use OT-path initialization, a learnable bandwidth initialized by Silverman's rule, fixed uniform mixture weights, the forward-KL data loss and a static potential with two hidden layers of $64$ units. We learn $c_H$ from $0.03$ and fix $c_W=0$. We set $\lambda$ once at initialization so that the data and residual losses have equal gradient norms with respect to the center paths.

    Stitching uses the velocity residual and Double-Stitch uses the Lagrangian residual. For Double-Stitch, we learn friction from $\gamma_0=10$ on the Ornstein--Uhlenbeck process and fix $\gamma=0$ and $1$ on the conservative and damped oscillators, respectively. We report the mean held-out $W_1$ between the fitted and analytic position densities. For visualization, we smooth both methods' fitted densities with $\sigma=0.1$; the scores use the unsmoothed densities.

    \paragraph{Synthetic cost experiment}
    \label{app:cost-setup}
    For \cref{fig:cost-efficiency}, we use the harmonic oscillator and initial position distribution from \cref{fig:gaussian-mixture}, with zero initial velocity. We sample $1000$ independent positions at nine observation times on $[0,1.35]$ and at eight held-out midpoints. Both methods use static potentials with $\gamma=c_H=c_W=0$ and $c_V=1$ fixed. We train for $5000$ Adam updates over three seeds sharing the same data.

    Double-Stitch uses $1000$ centers, $41$ time nodes and a potential with two hidden layers of $64$ SiLU units. We use OT-path initialization, a learnable bandwidth initialized at the coarsest scale, fixed uniform mixture weights, and the forward-KL data loss with the Lagrangian residual weighted by $\lambda=1$. WLM uses four attention layers with four heads, width $64$ and feed-forward width $512$. We train it with batches of $256$ and a random-horizon debiased Sinkhorn objective, without moving averages. The learning rates are constant at $10^{-3}$ for Double-Stitch and $10^{-4}$ for WLM, without weight decay. Double-Stitch's time-grid spacing $h$ and WLM's training integration step are both $0.03375$.

    Double-Stitch predicts directly from its KDE. WLM rolls out from the initial cloud with zero velocity and $50$ substeps between evaluation times. We average training $W_1$ over the eight positive observation times and test $W_1$ over the eight held-out midpoints. At common update counts, we plot the mean and std against the mean synchronized execution time of the forward, backward and optimizer operations.

\section{Additional results}
\label{app:additional-results}
\raggedbottom

\subsection{Identifiability}
\label{app:identifiability}

    Snapshots alone generally do not determine the underlying dynamics. Theoretical analysis of this question exists for SDEs: the dynamics are identifiable from the snapshots for (a) linear SDEs with additive noise whose initial distribution does not have generalized rotational symmetries \citep{guan2024}, and (b) gradient-flow SDEs with constant diffusivity observed away from equilibrium \citep{guan2026sde}. These results do not cover our setting: second-order (Lagrangian) dynamics with finitely many snapshots. We therefore provide examples that demonstrate the challenges of identifiability.
     
    \begin{figure}[H]
    \centering
    \captionsetup[subfigure]{font=figpanel,labelfont=normalfont}
    \begin{subfigure}[t]{.48\linewidth}
        \centering
        \caption{Canonical velocity: translation}
        \label{fig:canonical-velocity-translation}
        \includegraphics[width=\linewidth]{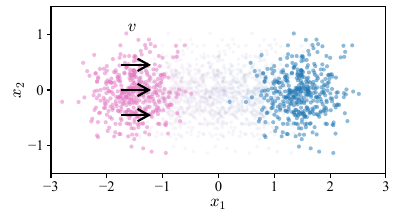}
    \end{subfigure}\hfill%
    \begin{subfigure}[t]{.48\linewidth}
        \centering
        \caption{General velocity: translation and rotation}
        \label{fig:general-velocity-rotation}
        \includegraphics[width=\linewidth]{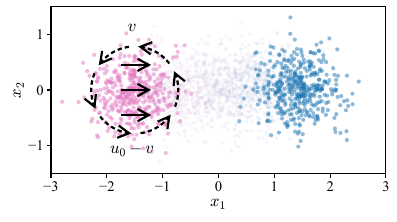}
    \end{subfigure}
    \par\vspace{-2pt}
    \includegraphics[width=.30\linewidth]{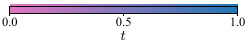}
    \caption{\textbf{Canonical and general velocities can produce the same snapshots.} Left: translation with the canonical velocity (solid arrows). Right: translation and rotation (dashed arcs) with a general velocity. Color indicates time.}
    \label{fig:canonical-velocity-transport}
\end{figure}

    \paragraph{Canonical and general velocities} Since we operate in Wasserstein spaces, our metric cannot differentiate between populations with the same density distribution but different velocity fields. For this reason, a canonical velocity is often considered. It is irrotational and has minimum kinetic energy among admissible fields \citep{ambrosio2008}. However, if the initial velocity is not canonical, the dynamical law must be able to take this into account, which is exactly what our Clebsch-based approach does (see \cref{app:clebsch-derivation}).

    In \cref{fig:canonical-velocity-transport} we show a Gaussian distribution that is transported under constant horizontal velocity from left to right ($\mathcal N ((-1.5, 0), 0.4^2 I)$, $400$ points). On the left its velocity is a pure translation ($u_0 = v_0 = (3,0)$). On the right it additionally has a rotational component ($\omega=3\pi/2$), meaning the particles experience an additional force. The density is the same at all times, however the dynamical law has to be written for $u$.

    \paragraph{Different initial velocities}
    Even with a known potential and friction, a finite set of position snapshots does not necessarily determine the initial velocity. We illustrate this with trajectories of a single particle following Lagrangian dynamics in a circular potential $V(x)=\tfrac12\omega_0^2\lVert x\rVert^2$ ($\omega_0=4\pi$, $\gamma=0.2\omega_0$). We set the particle at $x_0=(1,0)$ and consider two different initial velocities: $u_0^{(1)}=(0,-\omega_0)$ and $u_0^{(2)}=(0,-\omega_0/2)$. 
    
    It is easy to show that both trajectories pass through the same points: 
    \begin{equation*}
        x(t_k)=(-1)^k\mathrm{e}^{-\gamma t_k/2}x_0,
        \qquad t_k=\frac{k\pi}{\omega},\quad  \omega=\sqrt{\omega_0^2-\gamma^2/4}, \quad k=0,\ldots,4.
    \end{equation*}
    Therefore, if the observed snapshots are selected precisely as the intersection points of these trajectories (see \cref{fig:identifiability-v0}), the trajectories will not be identifiable.

    \begin{figure}[H]
    \centering
    \includegraphics[width=\linewidth]{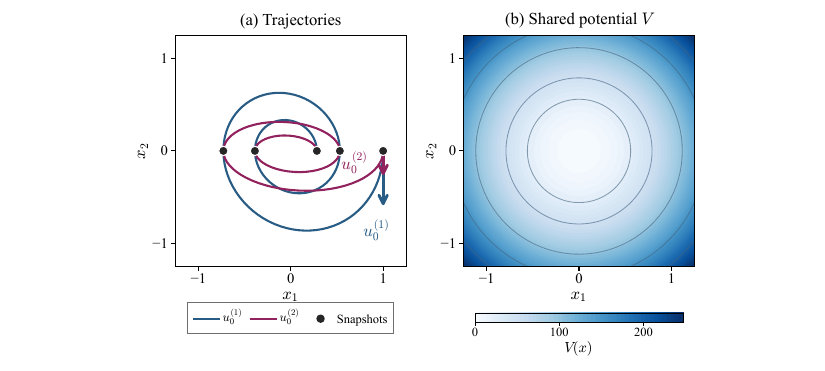}
    \caption{\textbf{Snapshots are not identifiable without the initial velocity.} Left: simulated trajectories with their individual initial velocities pass through the same points and are therefore not identifiable from them. Right: shared circular potential.}
    \label{fig:identifiability-v0}
\end{figure}

    \paragraph{Different static potentials}
    The situation is similar in the opposite case: when the initial conditions are the same, multiple static potentials can produce similar trajectories. 
    
    We again consider a single particle following Lagrangian dynamics, but now with a single initial condition ($x_0=(1,0)$ and $u_0=(-0.1,-1)$) and under the following potentials:
    \begin{equation*}
        V_{\mathrm{spiral}}(r,\theta)=K\bigl[1-\cos\bigl(\theta-\beta(r-r_0)\bigr)\bigr]+U(r),
        \qquad V_{\mathrm{circular}}(x)=\tfrac12\lambda\lVert x\rVert^2,
    \end{equation*}
    where $U'(r)=-a_0\sqrt{1+\beta^2r^2}$. We tune the parameters $K=0.3$, $\beta=15.47$, $r_0=0.96170$, $a_0=-0.07251$ and $\lambda=1.01$ using least squares to match the two trajectories.
    
    From \cref{fig:identifiability-potentials} we see that we are able to find different potentials that again produce nearly intersecting trajectories. Similar to the previous case, if the observed snapshots are at the intersection points, then the exact potential is non-identifiable.

    \begin{figure}[H]
    \centering
    \includegraphics[width=\linewidth]{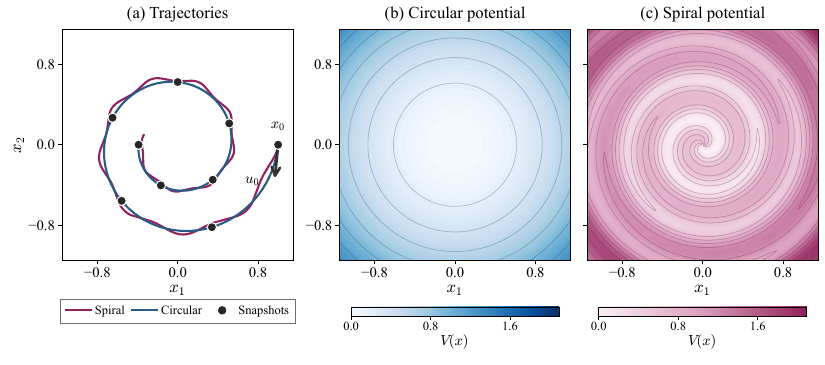}
    \caption{\textbf{Potentials are not identifiable from only snapshots.} Left: simulated trajectories with their shared initial velocities pass through the same points and are therefore not identifiable from them. Right: the potentials that produced the corresponding trajectories.}
    \label{fig:identifiability-potentials}
\end{figure}

\subsection{Single-cell dynamics}
\label{app:eb-results}

    \Cref{tab:eb-main,tab:eb-gap,tab:eb-hard} expand the three columns of \cref{tab:eb-comparison} to per-snapshot performance. Each table lists the previously reported results as well as our reruns. We rerun WLM and Stitching with original settings as well and compare them to the settings we used in \cref{app:reproduction-audit}.
    
    \begin{table}[H]
    \centering
    \setlength{\tabcolsep}{1.8pt}
    \caption{\textbf{Per-snapshot performance on EB LPO} ($W_1$, mean \textpm{} std, 5 seeds, lower is better).}
    \label{tab:eb-main}
    \appendixtablesize
    \begin{tabular}{@{}llrrrrr@{}}
\toprule
Method & Source / settings & \multicolumn{1}{c}{$t_1$} & \multicolumn{1}{c}{$t_2$} & \multicolumn{1}{c}{$t_3$} & \multicolumn{1}{c}{$t_4$} & \multicolumn{1}{c}{Mean} \\
\midrule
Neural SDE & \citealt{li2020} & $0.690\phantom{{}\pm 0.000}$ & $0.910\phantom{{}\pm 0.000}$ & $0.850\phantom{{}\pm 0.000}$ & $0.810\phantom{{}\pm 0.000}$ & $0.820\phantom{{}\pm 0.000}$ \\
TrajectoryNet & \citealt{tong2020} & $0.730\phantom{{}\pm 0.000}$ & $1.060\phantom{{}\pm 0.000}$ & $0.900\phantom{{}\pm 0.000}$ & $1.010\phantom{{}\pm 0.000}$ & $0.930\phantom{{}\pm 0.000}$ \\
SB-FBSDE & \citealt{chen2022} & $0.560\phantom{{}\pm 0.000}$ & $0.800\phantom{{}\pm 0.000}$ & $1.000\phantom{{}\pm 0.000}$ & $1.000\phantom{{}\pm 0.000}$ & $0.840\phantom{{}\pm 0.000}$ \\
NLSB & \citealt{koshizuka2023} & $0.680\phantom{{}\pm 0.000}$ & $0.840\phantom{{}\pm 0.000}$ & $0.810\phantom{{}\pm 0.000}$ & $0.790\phantom{{}\pm 0.000}$ & $0.780\phantom{{}\pm 0.000}$ \\
OT-CFM & \citealt{tong2024} & $0.780\phantom{{}\pm 0.000}$ & $0.760\phantom{{}\pm 0.000}$ & $0.770\phantom{{}\pm 0.000}$ & $0.750\phantom{{}\pm 0.000}$ & $0.770\phantom{{}\pm 0.000}$ \\
WLF-OT & \citealt{neklyudov2024} & $0.650\phantom{{}\pm 0.000}$ & $0.780\phantom{{}\pm 0.000}$ & $0.760\phantom{{}\pm 0.000}$ & $0.750\phantom{{}\pm 0.000}$ & $0.740\phantom{{}\pm 0.000}$ \\
WLF-SB & \citealt{neklyudov2024} & $0.630\phantom{{}\pm 0.000}$ & $0.790\phantom{{}\pm 0.000}$ & $0.770\phantom{{}\pm 0.000}$ & $0.740\phantom{{}\pm 0.000}$ & $0.730\phantom{{}\pm 0.000}$ \\
JKOnet & \citealt{bunne2022} & $1.530\phantom{{}\pm 0.000}$ & $1.270\phantom{{}\pm 0.000}$ & $1.130\phantom{{}\pm 0.000}$ & $1.410\phantom{{}\pm 0.000}$ & $1.340\phantom{{}\pm 0.000}$ \\
JKOnet$^*$, $V(x)$ & \citealt{terpin2024jkonetstar} & $0.990\phantom{{}\pm 0.000}$ & $1.110\phantom{{}\pm 0.000}$ & $1.060\phantom{{}\pm 0.000}$ & $1.300\phantom{{}\pm 0.000}$ & $1.120\phantom{{}\pm 0.000}$ \\
JKOnet$^*$, $V(x,t)$ & \citealt{terpin2024jkonetstar} & $0.690\phantom{{}\pm 0.000}$ & $0.770\phantom{{}\pm 0.000}$ & $0.690\phantom{{}\pm 0.000}$ & $0.780\phantom{{}\pm 0.000}$ & $0.730\phantom{{}\pm 0.000}$ \\
iJKOnet, $V(x)$ & \citealt{persiianov2025} & $0.920\phantom{{}\pm 0.000}$ & $1.110\phantom{{}\pm 0.000}$ & $0.950\phantom{{}\pm 0.000}$ & $1.210\phantom{{}\pm 0.000}$ & $1.050\phantom{{}\pm 0.000}$ \\
iJKOnet, $V(x,t)$ & \citealt{persiianov2025} & $0.510\phantom{{}\pm 0.000}$ & $0.580\phantom{{}\pm 0.000}$ & $0.570\phantom{{}\pm 0.000}$ & $0.640\phantom{{}\pm 0.000}$ & $0.580\phantom{{}\pm 0.000}$ \\
Stitching, $V(x)$ & \citealt{heinonen2026} & $0.460 \pm 0.010$ & $0.600 \pm 0.010$ & $0.600 \pm 0.010$ & $0.650 \pm 0.010$ & $0.580 \pm 0.010$ \\
Stitching, $V(x,t)$ & \citealt{heinonen2026} & $0.440 \pm 0.010$ & $0.560 \pm 0.010$ & $0.570 \pm 0.010$ & $0.600 \pm 0.020$ & $0.540 \pm 0.010$ \\
\grouprule
Stitching, $V(x)$ & Main settings & $0.509 \pm 0.003$ & $0.699 \pm 0.009$ & $0.681 \pm 0.017$ & $0.803 \pm 0.028$ & $0.673 \pm 0.005$ \\
Stitching, $V(x)$ & Original settings & $0.507 \pm 0.011$ & $0.650 \pm 0.017$ & $0.633 \pm 0.005$ & $0.682 \pm 0.010$ & $0.618 \pm 0.007$ \\
Stitching, $V(x,t)$ & Original settings & $0.490 \pm 0.012$ & $0.601 \pm 0.016$ & $0.595 \pm 0.013$ & $0.639 \pm 0.008$ & $0.581 \pm 0.010$ \\
WLM & Main settings & $0.660 \pm 0.018$ & $0.663 \pm 0.016$ & $0.617 \pm 0.028$ & $0.789 \pm 0.044$ & $0.682 \pm 0.009$ \\
Double-Stitch, $V(x)$ & Main settings & $0.546 \pm 0.005$ & $0.615 \pm 0.011$ & $0.676 \pm 0.016$ & $0.634 \pm 0.008$ & $0.618 \pm 0.001$ \\
Double-Stitch, $V(x,t)$ & Main settings & $0.484 \pm 0.005$ & $0.572 \pm 0.014$ & $0.582 \pm 0.007$ & $0.582 \pm 0.004$ & $0.555 \pm 0.005$ \\
\bottomrule
\end{tabular}

\end{table}

    \begin{table}[H]
    \centering
    \setlength{\tabcolsep}{1.8pt}
    \caption{\textbf{Per-snapshot performance on EB L1O} ($W_1$, mean \textpm{} std, 5 seeds, lower is better).}
    \label{tab:eb-gap}
    \appendixtablesize
    \begin{tabular}{@{}llrrrr@{}}
\toprule
Method & Source / settings & \multicolumn{1}{c}{$t_1$} & \multicolumn{1}{c}{$t_2$} & \multicolumn{1}{c}{$t_3$} & \multicolumn{1}{c}{Mean} \\
\midrule
TrajectoryNet & \citealt{kapusniak2024} & \tbdnum & \tbdnum & \tbdnum & $0.848\phantom{{}\pm 0.000}$ \\
NLSB & \citealt{kapusniak2024} & \tbdnum & \tbdnum & \tbdnum & $0.970\phantom{{}\pm 0.000}$ \\
OT-CFM & \citealt{kapusniak2024} & \tbdnum & \tbdnum & \tbdnum & $0.790 \pm 0.068$ \\
WLF-OT & \citealt{kapusniak2024} & \tbdnum & \tbdnum & \tbdnum & $0.742 \pm 0.012$ \\
WLF-SB & \citealt{kapusniak2024} & \tbdnum & \tbdnum & \tbdnum & $0.746 \pm 0.016$ \\
Reg. CNF & \citealt{kapusniak2024} & \tbdnum & \tbdnum & \tbdnum & $0.825 \pm 0.429$ \\
DSBM & \citealt{kapusniak2024} & \tbdnum & \tbdnum & \tbdnum & $1.775 \pm 0.429$ \\
DSB & \citealt{kapusniak2024} & \tbdnum & \tbdnum & \tbdnum & $0.862 \pm 0.023$ \\
SF2M-Sink & \citealt{kapusniak2024} & \tbdnum & \tbdnum & \tbdnum & $1.198 \pm 0.342$ \\
SF2M-Geo & \citealt{kapusniak2024} & \tbdnum & \tbdnum & \tbdnum & $0.879 \pm 0.148$ \\
SF2M-Exact & \citealt{kapusniak2024} & \tbdnum & \tbdnum & \tbdnum & $0.793 \pm 0.066$ \\
I-CFM & \citealt{kapusniak2024} & \tbdnum & \tbdnum & \tbdnum & $0.872 \pm 0.087$ \\
SB-CFM & \citealt{kapusniak2024} & \tbdnum & \tbdnum & \tbdnum & $1.221 \pm 0.380$ \\
WLF-UOT & \citealt{kapusniak2024} & \tbdnum & \tbdnum & \tbdnum & $0.738 \pm 0.014$ \\
I-MFM & \citealt{kapusniak2024} & \tbdnum & \tbdnum & \tbdnum & $0.822 \pm 0.042$ \\
OT-MFM & \citealt{kapusniak2024} & \tbdnum & \tbdnum & \tbdnum & $0.713 \pm 0.039$ \\
AM & \citealt{neklyudov2023} & \tbdnum & \tbdnum & \tbdnum & $0.944 \pm 0.003$ \\
UAM & \citealt{neklyudov2023} & \tbdnum & \tbdnum & \tbdnum & $0.924 \pm 0.004$ \\
sAM & \citealt{neklyudov2023} & \tbdnum & \tbdnum & \tbdnum & $0.939 \pm 0.003$ \\
DICE & \citealt{blickhan2025} & \tbdnum & \tbdnum & \tbdnum & $0.826 \pm 0.002$ \\
WLM & \citealt{guan2026wlm} & \tbdnum & \tbdnum & \tbdnum & $0.704 \pm 0.021$ \\
\grouprule
Stitching, $V(x)$ & Main settings & $0.844 \pm 0.046$ & $0.793 \pm 0.020$ & $0.822 \pm 0.012$ & $0.819 \pm 0.021$ \\
Stitching, $V(x)$ & Original settings & $0.738 \pm 0.038$ & $0.807 \pm 0.019$ & $0.844 \pm 0.014$ & $0.796 \pm 0.014$ \\
Stitching, $V(x,t)$ & Original settings & $0.654 \pm 0.023$ & $0.726 \pm 0.023$ & $0.769 \pm 0.007$ & $0.717 \pm 0.010$ \\
WLM & Main settings & $0.713 \pm 0.019$ & $0.753 \pm 0.060$ & $0.735 \pm 0.020$ & $0.734 \pm 0.023$ \\
Double-Stitch, $V(x)$ & Main settings & $0.692 \pm 0.012$ & $0.703 \pm 0.011$ & $0.805 \pm 0.004$ & $0.733 \pm 0.004$ \\
Double-Stitch, $V(x,t)$ & Main settings & $0.695 \pm 0.006$ & $0.704 \pm 0.012$ & $0.754 \pm 0.007$ & $0.717 \pm 0.008$ \\
\bottomrule
\end{tabular}

\end{table}

    \begin{table}[H]
    \centering
    \setlength{\tabcolsep}{1.8pt}
    \caption{\textbf{Per-snapshot performance on EB L2O} ($W_2$, mean \textpm{} std, 5 seeds, lower is better).}
    \label{tab:eb-hard}
    \appendixtablesize
    \begin{tabular}{@{}llrrr@{}}
\toprule
Method & Source / settings & \multicolumn{1}{c}{$t_1$} & \multicolumn{1}{c}{$t_3$} & \multicolumn{1}{c}{Mean} \\
\midrule
TrajectoryNet & \citealt{tong2020} & $2.030 \pm 0.040$ & $1.930 \pm 0.080$ & $1.980\phantom{{}\pm 0.000}$ \\
Vanilla-SB & \citealt{vargas2021} & $1.490 \pm 0.060$ & $1.550 \pm 0.030$ & $1.520\phantom{{}\pm 0.000}$ \\
MMSB & \citealt{shen2025} & $1.270 \pm 0.030$ & $1.570 \pm 0.050$ & $1.420\phantom{{}\pm 0.000}$ \\
DMSB & \citealt{chen2023} & $1.130 \pm 0.080$ & $1.450 \pm 0.160$ & $1.290\phantom{{}\pm 0.000}$ \\
JKOnet$^*$, $V(x)$ & \citealt{terpin2024jkonetstar} & $1.150 \pm 0.030$ & $2.530 \pm 0.010$ & $1.840\phantom{{}\pm 0.000}$ \\
JKOnet$^*$, $V(x,t)$ & \citealt{terpin2024jkonetstar} & $4.410 \pm 1.500$ & $2.770 \pm 0.200$ & $3.590\phantom{{}\pm 0.000}$ \\
iJKOnet, $V(x)$ & \citealt{persiianov2025} & $1.080 \pm 0.010$ & $1.150 \pm 0.000$ & $1.120\phantom{{}\pm 0.000}$ \\
iJKOnet, $V(x,t)$ & \citealt{persiianov2025} & $0.980 \pm 0.040$ & $0.850 \pm 0.020$ & $0.920\phantom{{}\pm 0.000}$ \\
Stitching, $V(x)$ & \citealt{heinonen2026} & $0.850 \pm 0.060$ & $0.990 \pm 0.030$ & $0.920 \pm 0.030$ \\
Stitching, $V(x,t)$ & \citealt{heinonen2026} & $0.840 \pm 0.050$ & $0.920 \pm 0.030$ & $0.880 \pm 0.040$ \\
\grouprule
Stitching, $V(x)$ & Main settings & $0.797 \pm 0.022$ & $0.893 \pm 0.007$ & $0.845 \pm 0.014$ \\
Stitching, $V(x)$ & Original settings & $0.752 \pm 0.037$ & $0.891 \pm 0.007$ & $0.821 \pm 0.019$ \\
Stitching, $V(x,t)$ & Original settings & $0.716 \pm 0.038$ & $0.856 \pm 0.025$ & $0.786 \pm 0.008$ \\
WLM & Main settings & $0.846 \pm 0.045$ & $0.800 \pm 0.120$ & $0.823 \pm 0.041$ \\
Double-Stitch, $V(x)$ & Main settings & $0.755 \pm 0.008$ & $0.832 \pm 0.004$ & $0.794 \pm 0.004$ \\
Double-Stitch, $V(x,t)$ & Main settings & $0.769 \pm 0.004$ & $0.829 \pm 0.008$ & $0.799 \pm 0.003$ \\
\bottomrule
\end{tabular}

\end{table}

    In \cref{fig:eb-prediction,fig:eb-l2o-prediction} we show Double-Stitch predictions for L1O and L2O scenarios.

    \begin{figure}[H]
    \centering
    \includegraphics[width=\linewidth]{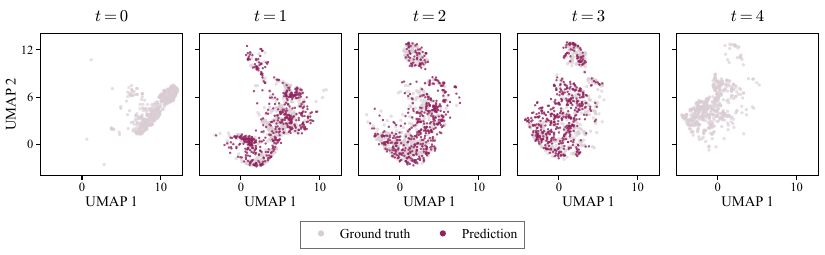}
    \caption{\textbf{Double-Stitch predictions on EB L1O} ($V(x)$, seed $0$). Shared UMAP. Pale points are ground truth, dark points are predictions.}
    \label{fig:eb-prediction}
\end{figure}

    \begin{figure}[H]
    \centering
    \includegraphics[width=\linewidth]{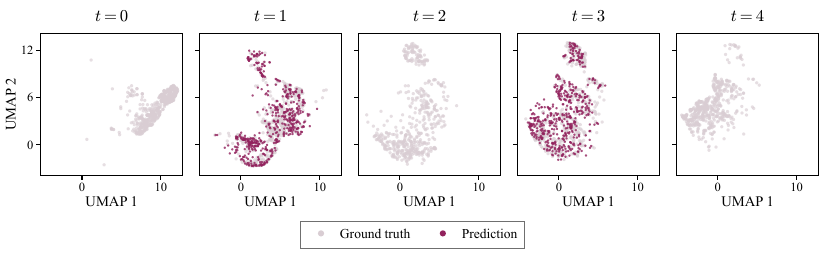}
    \caption{\textbf{Double-Stitch predictions on EB L2O} ($V(x)$, seed $0$). Shared UMAP. Pale points are ground truth, dark points are predictions.}
    \label{fig:eb-l2o-prediction}
\end{figure}

\subsection{Ocean vortex dynamics}
\label{app:oceans-results}

    In \cref{tab:hycom-main,tab:hycom-forecast} we show per-snapshot results for interpolation (oceans small vortex) and the performance on three forecasting protocols for extrapolation (oceans large vortex).

    \begin{table}[H]
    \centering
    \setlength{\tabcolsep}{1.8pt}
    \caption{\textbf{Per-snapshot performance on the small ocean vortex} ($W_1$, mean \textpm{} std, 3 seeds, lower is better).}
    \label{tab:hycom-main}
    \appendixtablesize
    \begin{tabular}{@{}llrrrrr@{}}
\toprule
Method & Source / settings & \multicolumn{1}{c}{$t_1$} & \multicolumn{1}{c}{$t_3$} & \multicolumn{1}{c}{$t_5$} & \multicolumn{1}{c}{$t_7$} & \multicolumn{1}{c}{Mean} \\
\midrule
AM & \citealt{neklyudov2023} & $0.358 \pm 0.009$ & $0.353 \pm 0.020$ & $0.447 \pm 0.013$ & $0.346 \pm 0.016$ & $0.376\phantom{{}\pm 0.000}$ \\
UAM & \citealt{neklyudov2023} & $0.291 \pm 0.008$ & $0.378 \pm 0.021$ & $0.514 \pm 0.010$ & $0.349 \pm 0.009$ & $0.383\phantom{{}\pm 0.000}$ \\
sAM & \citealt{neklyudov2023} & $0.358 \pm 0.009$ & $0.354 \pm 0.010$ & $0.451 \pm 0.016$ & $0.350 \pm 0.009$ & $0.378\phantom{{}\pm 0.000}$ \\
DICE & \citealt{blickhan2025} & $0.369 \pm 0.019$ & $0.252 \pm 0.020$ & $0.240 \pm 0.011$ & $0.115 \pm 0.009$ & $0.244\phantom{{}\pm 0.000}$ \\
Vanilla-SB & \citealt{vargas2021} & $0.112 \pm 0.061$ & $0.178 \pm 0.073$ & $0.262 \pm 0.043$ & $0.305 \pm 0.032$ & $0.214\phantom{{}\pm 0.000}$ \\
OT-CFM & \citealt{tong2024} & $0.148 \pm 0.004$ & $0.227 \pm 0.008$ & $0.191 \pm 0.012$ & $0.250 \pm 0.018$ & $0.204\phantom{{}\pm 0.000}$ \\
OT-MFM & \citealt{kapusniak2024} & $0.107 \pm 0.014$ & $0.056 \pm 0.014$ & $0.052 \pm 0.011$ & $0.070 \pm 0.021$ & $0.071\phantom{{}\pm 0.000}$ \\
Curly-FM & \citealt{petrovic2025curlyfm} & $0.019 \pm 0.003$ & $0.045 \pm 0.005$ & $0.027 \pm 0.001$ & $0.030 \pm 0.006$ & $0.030\phantom{{}\pm 0.000}$ \\
SBIRR & \citealt{shen2025} & $0.073 \pm 0.035$ & $0.087 \pm 0.033$ & $0.062 \pm 0.010$ & $0.082 \pm 0.022$ & $0.076\phantom{{}\pm 0.000}$ \\
Snap-MMD & \citealt{berlinghieri2025} & $0.051 \pm 0.003$ & $0.067 \pm 0.002$ & $0.131 \pm 0.005$ & $0.056 \pm 0.001$ & $0.076\phantom{{}\pm 0.000}$ \\
WLM, published & \citealt{guan2026wlm} & $0.039 \pm 0.002$ & $0.028 \pm 0.000$ & $0.083 \pm 0.002$ & $0.050 \pm 0.004$ & $0.050\phantom{{}\pm 0.000}$ \\
\grouprule
Stitching, $V(x)$ & Main settings & $1.538 \pm 0.325$ & $1.490 \pm 0.331$ & $1.431 \pm 0.263$ & $1.601 \pm 0.177$ & $1.515 \pm 0.273$ \\
Stitching, $V(x,t)$ & Main settings & $0.176 \pm 0.043$ & $0.158 \pm 0.038$ & $0.144 \pm 0.031$ & $0.187 \pm 0.028$ & $0.166 \pm 0.035$ \\
Stitching, $V(x)$ & Original settings & $8.180 \pm 0.643$ & $9.638 \pm 1.418$ & $0.475 \pm 0.115$ & $0.705 \pm 0.046$ & $4.750 \pm 0.535$ \\
Stitching, $V(x,t)$ & Original settings & $0.080 \pm 0.031$ & $0.059 \pm 0.019$ & $0.099 \pm 0.068$ & $0.178 \pm 0.026$ & $0.104 \pm 0.020$ \\
WLM & Main settings & $0.056 \pm 0.027$ & $0.048 \pm 0.012$ & $0.086 \pm 0.019$ & $0.065 \pm 0.039$ & $0.064 \pm 0.015$ \\
WLM & Original settings & $0.037 \pm 0.002$ & $0.030 \pm 0.005$ & $0.087 \pm 0.015$ & $0.083 \pm 0.064$ & $0.059 \pm 0.015$ \\
Double-Stitch, $V(x)$ & Main settings & $0.068 \pm 0.023$ & $0.033 \pm 0.004$ & $0.087 \pm 0.010$ & $0.070 \pm 0.043$ & $0.065 \pm 0.004$ \\
Double-Stitch, $V(x,t)$ & Main settings & $0.047 \pm 0.025$ & $0.037 \pm 0.003$ & $0.132 \pm 0.015$ & $0.076 \pm 0.007$ & $0.073 \pm 0.009$ \\
\bottomrule
\end{tabular}

\end{table}

    \begin{table}[H]
    \centering
    \setlength{\tabcolsep}{1.8pt}
    \caption{\textbf{Per-snapshot performance on the large ocean vortex} ($W_1$, mean \textpm{} std, 3 seeds, lower is better). Dashes indicate unreported or unavailable cells.}
    \label{tab:hycom-forecast}
    \appendixtablesize
    \begin{tabular}{@{}llrrr@{}}
\toprule
Method & Source / settings & \multicolumn{1}{c}{Rollout $0\!\to\!10$} & \multicolumn{1}{c}{Pushfwd. $9\!\to\!10$} & \multicolumn{1}{c}{Direct} \\
\midrule
SBIRR & \citealt{shen2025} & $\phantom{0}1.062 \pm 0.111$ & $0.567 \pm 0.014$ & \tbdnum \\
Snap-MMD & \citealt{berlinghieri2025} & $\phantom{0}0.896 \pm 0.141$ & $0.473 \pm 0.013$ & \tbdnum \\
WLM, published & \citealt{guan2026wlm} & $\phantom{0}0.180 \pm 0.038$ & $0.065 \pm 0.005$ & \tbdnum \\
\grouprule
Stitching, $V(x)$ & Main settings & $14.370 \pm 6.123$ & $1.388 \pm 0.119$ & $0.295 \pm 0.032$ \\
Stitching, $V(x,t)$ & Main settings & $\phantom{0}0.224 \pm 0.017$ & $0.195 \pm 0.022$ & $0.212 \pm 0.018$ \\
Stitching, $V(x)$ & Original settings & $\phantom{0}0.149 \pm 0.007$ & $0.205 \pm 0.006$ & $0.206 \pm 0.001$ \\
Stitching, $V(x,t)$ & Original settings & $\phantom{0}0.211 \pm 0.092$ & $0.170 \pm 0.004$ & $0.175 \pm 0.008$ \\
WLM & Main settings & $\phantom{0}0.194 \pm 0.048$ & $0.088 \pm 0.023$ & \tbdnum \\
WLM & Original settings & $\phantom{0}0.889 \pm 1.126$ & $0.238 \pm 0.275$ & \tbdnum \\
Double-Stitch, $V(x)$ & Main settings & $\phantom{0}1.839 \pm 0.539$ & $0.134 \pm 0.003$ & $0.209 \pm 0.013$ \\
Double-Stitch, $V(x,t)$ & Main settings & $\phantom{0}0.222 \pm 0.008$ & $0.119 \pm 0.019$ & $0.194 \pm 0.007$ \\
\bottomrule
\end{tabular}

\end{table}

    In \cref{fig:oceans-prediction} we show Double-Stitch predictions for the small vortex interpolation.

    \begin{figure}[H]
    \centering
    \includegraphics[width=\linewidth]{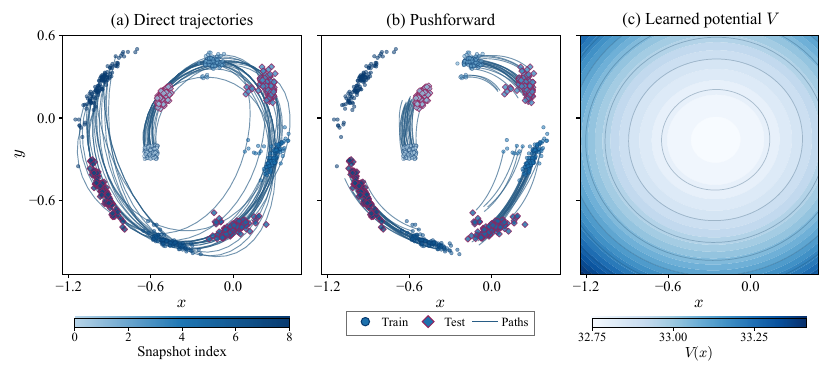}
    \caption{\textbf{Double-Stitch predictions on the small ocean vortex} ($V(x)$, seed $0$). Left: learned trajectories (used for direct evaluation). Center: pushforward predictions (used for pushforward evaluation). Right: learned potential.}
    \label{fig:oceans-prediction}
\end{figure}

\subsection{Synthetic demonstration}
\label{app:synthetic-results}

    \Cref{fig:synthetic-density-comparison} shows the ground truth and fitted densities for the three Gaussian mixture flows (see \cref{sec:demonstrations}). Each method uses the same training snapshots and fitting budget (see the settings in \cref{app:training-details}).

    \begin{figure}[H]
    \centering
    \includegraphics[width=\linewidth]{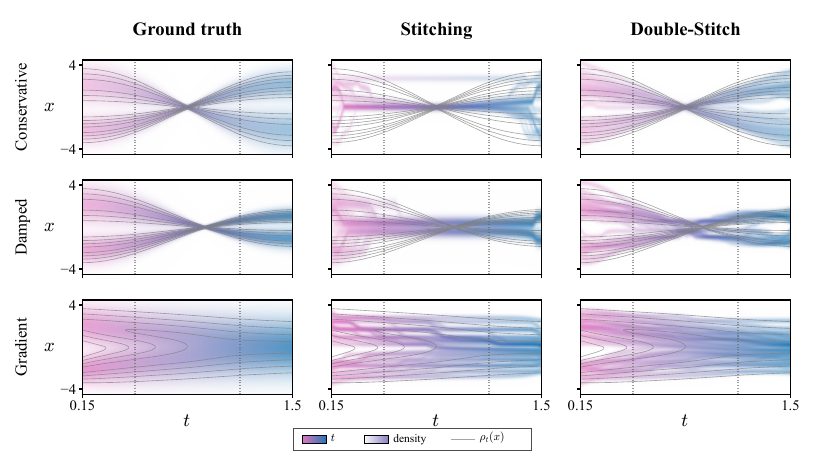}
    \caption{\textbf{Double-Stitch learns all three Gaussian-mixture flows.} Left: ground-truth conservative, damped and gradient densities. Center: Stitching, which fails on the conservative flow and underperforms on the damped flow. Right: Double-Stitch, which succeeds on all three. Dotted lines mark the held-out times.}
    \label{fig:synthetic-density-comparison}
\end{figure}

\subsection{Ablations and sensitivity studies}
\label{app:sensitivity}

    \paragraph{Prediction mode and settings} We compare the performance of Stitching, WLM, and Double-Stitch on the ocean datasets across all prediction modes in \cref{tab:prediction-protocol}. The main and original settings of each method are described in \cref{app:reproduction-audit}. We find that Double-Stitch has similar performance under direct and pushforward evaluations; rollout is worse, but surprisingly improves if a time-dependent potential is used. Stitching scores well under direct evaluation, but not under the other two, likely because its first-order condition is essentially misspecified on this data. For WLM, the original settings are slightly better on interpolation but worse on forecasting.

    \begin{table}[H]
    \centering
    \setlength{\tabcolsep}{3pt}
    \caption{\textbf{Double-Stitch and WLM sensitivity to the prediction mode on ocean vortices} ($W_1$, mean \textpm{} std, 3 seeds, lower is better). Dashes indicate unavailable cells.}
    \label{tab:prediction-protocol}
    \appendixtablesize
    \begin{tabular}{@{}llrrr@{}}
\toprule
Method & Settings & \multicolumn{1}{c}{Direct} & \multicolumn{1}{c}{Pushforward} & \multicolumn{1}{c}{Rollout} \\
\midrule
\multicolumn{5}{l}{\emph{Interpolation}} \\
Stitching, $V(x)$ & Main settings & $0.366 \pm 0.015$ & $1.515 \pm 0.273$ & $\phantom{0}3.858 \pm 0.912$ \\
Stitching, $V(x,t)$ & Main settings & $0.125 \pm 0.022$ & $0.166 \pm 0.035$ & $1.089 \pm 0.475$ \\
Stitching, $V(x)$ & Original settings & $0.301 \pm 0.016$ & $4.750 \pm 0.535$ & $32.764 \pm 2.684$ \\
Stitching, $V(x,t)$ & Original settings & $0.104 \pm 0.021$ & $0.104 \pm 0.020$ & $\phantom{0}0.104 \pm 0.017$ \\
WLM & Main settings & \tbdnum & $0.064 \pm 0.015$ & $\phantom{0}0.176 \pm 0.098$ \\
WLM & Original settings & \tbdnum & $0.059 \pm 0.015$ & $\phantom{0}0.115 \pm 0.019$ \\
Double-Stitch, $V(x)$ & Main settings & $0.071 \pm 0.019$ & $0.065 \pm 0.004$ & $\phantom{0}0.270 \pm 0.050$ \\
Double-Stitch, $V(x,t)$ & Main settings & $0.063 \pm 0.011$ & $0.073 \pm 0.009$ & $0.103 \pm 0.047$ \\
\midrule
\multicolumn{5}{l}{\emph{Forecast}} \\
Stitching, $V(x)$ & Main settings & $0.295 \pm 0.032$ & $1.388 \pm 0.119$ & $14.370 \pm 6.123$ \\
Stitching, $V(x,t)$ & Main settings & $0.212 \pm 0.018$ & $0.195 \pm 0.022$ & $\phantom{0}0.224 \pm 0.017$ \\
Stitching, $V(x)$ & Original settings & $0.206 \pm 0.001$ & $0.205 \pm 0.006$ & $\phantom{0}0.149 \pm 0.007$ \\
Stitching, $V(x,t)$ & Original settings & $0.175 \pm 0.008$ & $0.170 \pm 0.004$ & $\phantom{0}0.211 \pm 0.092$ \\
WLM & Main settings & \tbdnum & $0.088 \pm 0.023$ & $\phantom{0}0.194 \pm 0.048$ \\
WLM & Original settings & \tbdnum & $0.238 \pm 0.275$ & $\phantom{0}0.889 \pm 1.126$ \\
Double-Stitch, $V(x)$ & Main settings & $0.209 \pm 0.013$ & $0.134 \pm 0.003$ & $\phantom{0}1.839 \pm 0.539$ \\
Double-Stitch, $V(x,t)$ & Main settings & $0.194 \pm 0.007$ & $0.119 \pm 0.019$ & $0.222 \pm 0.008$ \\
\bottomrule
\end{tabular}

\end{table}

    \paragraph{Training time} We investigate how the training time of Double-Stitch depends on the number of particles in \cref{tab:scaling}. We find that the training time increases 4--5 times across all tasks when the number of particles increases tenfold.

    \begin{table}[H]
    \centering
    \setlength{\tabcolsep}{1.8pt}
    \caption{\textbf{Double-Stitch training time versus number of particles across datasets} (seconds, mean \textpm{} std, 3 seeds).}
    \label{tab:scaling}
    \appendixtablesize
    \begin{tabular}{@{}p{0.23\linewidth}*{5}{>{\raggedleft\arraybackslash}p{0.135\linewidth}}@{}}
\toprule
 & \multicolumn{3}{c}{EB} & \multicolumn{2}{c}{Oceans} \\
\cmidrule(lr){2-4}\cmidrule(lr){5-6}
Setting & \multicolumn{1}{c}{LPO} & \multicolumn{1}{c}{L1O} & \multicolumn{1}{c}{L2O} & \multicolumn{1}{c}{Interp.} & \multicolumn{1}{c}{Forecast} \\
\midrule
$P=100$ & $28.8 \pm 0.2$ & $84.7 \pm 3.3$ & $29.0 \pm 2.3$ & $104.7 \pm 1.0$ & $104.4 \pm 3.0$ \\
$P=400$ & $63.3 \pm 0.3$ & $186.8 \pm 2.1$ & $62.8 \pm 0.6$ & $252.7 \pm 6.0$ & $277.1 \pm 2.9$ \\
$P=1000$ & $127.3 \pm 2.9$ & $385.5 \pm 4.9$ & $125.6 \pm 0.4$ & $512.4 \pm 16.9$ & $532.5 \pm 7.5$ \\
\bottomrule
\end{tabular}

\end{table}

    \paragraph{Initialization}
    \label{app:initialization}
    In \cref{tab:initialization} we compare the performance of Double-Stitch with different trajectory and bandwidth initializations (see definitions in \cref{app:training-details}). Trajectory initialization has a large effect: endpoint OT has the largest error on all five tasks, and McCann-path is generally close to OT-path except for interpolation on oceans. Bandwidth initialization has a smaller effect: the coarsest initialization is better on ocean interpolation, and the two perform similarly on the other tasks.
    
    \begin{table}[H]
    \centering
    \setlength{\tabcolsep}{3pt}
    \caption{\textbf{Double-Stitch sensitivity to initialization across datasets} (mean \textpm{} std, 3 seeds, lower is better).}
    \label{tab:initialization}
    \appendixtablesize
    \begin{tabular}{@{}p{0.23\linewidth}*{5}{>{\raggedleft\arraybackslash}p{0.135\linewidth}}@{}}
\toprule
 & \multicolumn{3}{c}{EB} & \multicolumn{2}{c}{Oceans} \\
\cmidrule(lr){2-4}\cmidrule(lr){5-6}
Setting & \multicolumn{1}{c}{LPO $W_1$} & \multicolumn{1}{c}{L1O $W_1$} & \multicolumn{1}{c}{L2O $W_2$} & \multicolumn{1}{c}{Interp. $W_1$} & \multicolumn{1}{c}{Forecast $W_1$} \\
\midrule
OT-path, coarsest (ref.) & $0.618 \pm 0.001$ & $0.733 \pm 0.003$ & $0.795 \pm 0.006$ & $0.065 \pm 0.004$ & $0.134 \pm 0.003$ \\
OT-path, Silverman & $0.614 \pm 0.001$ & $0.729 \pm 0.004$ & $0.795 \pm 0.010$ & $0.077 \pm 0.021$ & $0.131 \pm 0.013$ \\
Endpoint OT, Silverman & $0.777 \pm 0.004$ & $0.814 \pm 0.011$ & $0.814 \pm 0.011$ & $0.504 \pm 0.274$ & $0.578 \pm 0.056$ \\
McCann-path, Silverman & $0.640 \pm 0.004$ & $0.736 \pm 0.001$ & $0.791 \pm 0.002$ & $0.121 \pm 0.009$ & $0.127 \pm 0.009$ \\
\bottomrule
\end{tabular}

\end{table}

    \paragraph{Time-dependent potential} In \cref{tab:time-potential} we compare the performance of Double-Stitch with static and time-dependent potentials. Overall, the time-dependent potential outperforms the static potential on EB LPO and ocean forecasting and matches it elsewhere.

    \begin{table}[H]
    \centering
    \setlength{\tabcolsep}{3pt}
    \caption{\textbf{Double-Stitch sensitivity to a time-dependent potential across datasets} (mean \textpm{} std, 3 seeds, lower is better).}
    \label{tab:time-potential}
    \appendixtablesize
    \begin{tabular}{@{}p{0.23\linewidth}*{5}{>{\raggedleft\arraybackslash}p{0.135\linewidth}}@{}}
\toprule
 & \multicolumn{3}{c}{EB} & \multicolumn{2}{c}{Oceans} \\
\cmidrule(lr){2-4}\cmidrule(lr){5-6}
Setting & \multicolumn{1}{c}{LPO $W_1$} & \multicolumn{1}{c}{L1O $W_1$} & \multicolumn{1}{c}{L2O $W_2$} & \multicolumn{1}{c}{Interp. $W_1$} & \multicolumn{1}{c}{Forecast $W_1$} \\
\midrule
$V(x)$ & $0.618 \pm 0.001$ & $0.733 \pm 0.003$ & $0.794 \pm 0.005$ & $0.065 \pm 0.004$ & $0.134 \pm 0.003$ \\
$V(x,t)$ & $0.558 \pm 0.004$ & $0.716 \pm 0.007$ & $0.800 \pm 0.004$ & $0.073 \pm 0.009$ & $0.119 \pm 0.019$ \\
\bottomrule
\end{tabular}

\end{table}

    \paragraph{Residual-loss weight} In \cref{tab:residual-coefficient} we compare the performance with various residual-loss weights. A smaller weight is slightly better on LPO, L1O and ocean interpolation, while the larger weight shows better results on L2O and forecasting. We hypothesize that this happens because on L2O and forecasting the gaps between observed snapshots are larger and the model relies more on the dynamics.

    \begin{table}[H]
    \centering
    \setlength{\tabcolsep}{3pt}
    \caption{\textbf{Double-Stitch sensitivity to the residual-loss weight across datasets} (mean \textpm{} std, 3 seeds, lower is better).}
    \label{tab:residual-coefficient}
    \appendixtablesize
    \begin{tabular}{@{}p{0.23\linewidth}*{5}{>{\raggedleft\arraybackslash}p{0.135\linewidth}}@{}}
\toprule
 & \multicolumn{3}{c}{EB} & \multicolumn{2}{c}{Oceans} \\
\cmidrule(lr){2-4}\cmidrule(lr){5-6}
Setting & \multicolumn{1}{c}{LPO $W_1$} & \multicolumn{1}{c}{L1O $W_1$} & \multicolumn{1}{c}{L2O $W_2$} & \multicolumn{1}{c}{Interp. $W_1$} & \multicolumn{1}{c}{Forecast $W_1$} \\
\midrule
$\lambda=0.3$ & $0.591 \pm 0.005$ & $0.724 \pm 0.005$ & $0.805 \pm 0.005$ & $0.061 \pm 0.002$ & $0.134 \pm 0.008$ \\
$\lambda=1$ (ref.) & $0.618 \pm 0.001$ & $0.733 \pm 0.003$ & $0.795 \pm 0.006$ & $0.065 \pm 0.004$ & $0.134 \pm 0.003$ \\
$\lambda=3$ & $0.633 \pm 0.003$ & $0.739 \pm 0.006$ & $0.786 \pm 0.004$ & $0.071 \pm 0.011$ & $0.129 \pm 0.004$ \\
\bottomrule
\end{tabular}

\end{table}

    \paragraph{Initial friction} In \cref{tab:initial-friction} we compare the performance of Double-Stitch with various initial friction parameters. $\gamma_0=0.1$ seems to produce slightly lower errors than $\gamma_0=1$ on all five tasks, while $\gamma_0=10$ increases the errors on EB and ocean interpolation. For the main experiments we use $\gamma_0=1$ (fixed before the sensitivity study, as are the other parameters).

    \begin{table}[H]
    \centering
    \setlength{\tabcolsep}{3pt}
    \caption{\textbf{Double-Stitch sensitivity to the initial friction across datasets} (mean \textpm{} std, 3 seeds, lower is better).}
    \label{tab:initial-friction}
    \appendixtablesize
    \begin{tabular}{@{}p{0.23\linewidth}*{5}{>{\raggedleft\arraybackslash}p{0.135\linewidth}}@{}}
\toprule
 & \multicolumn{3}{c}{EB} & \multicolumn{2}{c}{Oceans} \\
\cmidrule(lr){2-4}\cmidrule(lr){5-6}
Setting & \multicolumn{1}{c}{LPO $W_1$} & \multicolumn{1}{c}{L1O $W_1$} & \multicolumn{1}{c}{L2O $W_2$} & \multicolumn{1}{c}{Interp. $W_1$} & \multicolumn{1}{c}{Forecast $W_1$} \\
\midrule
$\gamma_0=0.1$ & $0.610 \pm 0.003$ & $0.728 \pm 0.004$ & $0.793 \pm 0.003$ & $0.062 \pm 0.001$ & $0.113 \pm 0.018$ \\
$\gamma_0=1$ (ref.) & $0.618 \pm 0.001$ & $0.733 \pm 0.003$ & $0.795 \pm 0.006$ & $0.065 \pm 0.004$ & $0.134 \pm 0.003$ \\
$\gamma_0=10$ & $0.740 \pm 0.003$ & $0.812 \pm 0.003$ & $0.827 \pm 0.003$ & $0.148 \pm 0.005$ & $0.127 \pm 0.007$ \\
\bottomrule
\end{tabular}

\end{table}

    \paragraph{Number of particles} In \cref{tab:particle-count} we compare the performance of Double-Stitch with various numbers of particles. On EB, increasing the number of particles reduces the error, with little change on L2O beyond $P=400$. On oceans the situation is different: the error on the interpolation task increases with $P$. This is likely because this dataset has only $111$ particles per snapshot, so $P=100$ already matches its size. On forecasting, $P=1000$ gives the lowest error but needs the highest training time (see \cref{tab:scaling}).

    \begin{table}[H]
    \centering
    \setlength{\tabcolsep}{3pt}
    \caption{\textbf{Double-Stitch sensitivity to the number of particles across datasets} (mean \textpm{} std, 3 seeds, lower is better).}
    \label{tab:particle-count}
    \appendixtablesize
    \begin{tabular}{@{}p{0.23\linewidth}*{5}{>{\raggedleft\arraybackslash}p{0.135\linewidth}}@{}}
\toprule
 & \multicolumn{3}{c}{EB} & \multicolumn{2}{c}{Oceans} \\
\cmidrule(lr){2-4}\cmidrule(lr){5-6}
Setting & \multicolumn{1}{c}{LPO $W_1$} & \multicolumn{1}{c}{L1O $W_1$} & \multicolumn{1}{c}{L2O $W_2$} & \multicolumn{1}{c}{Interp. $W_1$} & \multicolumn{1}{c}{Forecast $W_1$} \\
\midrule
$P=100$ & $0.667 \pm 0.013$ & $0.766 \pm 0.021$ & $0.833 \pm 0.003$ & $0.065 \pm 0.004$ & $0.122 \pm 0.018$ \\
$P=400$ & $0.629 \pm 0.009$ & $0.735 \pm 0.002$ & $0.794 \pm 0.003$ & $0.079 \pm 0.004$ & $0.134 \pm 0.003$ \\
$P=1000$ & $0.618 \pm 0.001$ & $0.733 \pm 0.003$ & $0.795 \pm 0.006$ & $0.099 \pm 0.013$ & $0.098 \pm 0.008$ \\
\bottomrule
\end{tabular}

\end{table}

\end{document}